\pdfoutput=1
\pdfoutput=1
\documentclass[a4paper,11pt]{article}
\usepackage[margin=0.9in]{geometry}
\usepackage{amsmath,amsthm}
\allowdisplaybreaks
\usepackage[dvipsnames]{xcolor}
\usepackage{algorithm}
\usepackage[hidelinks]{hyperref}
\hypersetup{colorlinks,linkcolor={red!80!black},
  citecolor={blue!80!black},urlcolor={ForestGreen}}
\usepackage[capitalize,noabbrev]{cleveref}
\newtheorem{theorem}{Theorem}[section]
\newtheorem{lemma}[theorem]{Lemma}
\newtheorem{proposition}[theorem]{Proposition}

\theoremstyle{definition}

\theoremstyle{remark}
\newtheorem{remark}[theorem]{Remark}

\AddToHook{env/lemma/begin}{\crefalias{theorem}{lemma}}
\AddToHook{env/proposition/begin}{\crefalias{theorem}{proposition}}
\AddToHook{env/corollary/begin}{\crefalias{theorem}{corollary}}
\AddToHook{env/definition/begin}{\crefalias{theorem}{definition}}
\AddToHook{env/remark/begin}{\crefalias{theorem}{remark}}
\AddToHook{env/example/begin}{\crefalias{theorem}{example}}
\numberwithin{equation}{section}
\numberwithin{algorithm}{section}
\crefname{figure}{Figure}{Figures}
\Crefname{figure}{Figure}{Figures}
\crefname{algorithm}{Algorithm}{Algorithms}
\Crefname{algorithm}{Algorithm}{Algorithms}
\newenvironment{keywords}{\par\small\noindent\textbf{Keywords.} }{\par}
\newenvironment{MSCcodes}{\par\small\noindent\textbf{MSC codes.} }{\par}
\let\preprintparagraph\paragraph
\renewcommand{\paragraph}{%
  \par
  \ifdim\dimexpr\pagegoal-\pagetotal\relax<5\baselineskip
    \newpage
  \fi
  \preprintparagraph
}
\makeatletter
\renewcommand{\@maketitle}{%
  \vbox{%
    \vspace*{2cm}%
    \centering
    {\LARGE\bfseries\@title\par}%
    \vskip 1.5em
    {\large\begin{tabular}[t]{c}\@author\end{tabular}\par}%
    \vskip 0.3in minus 0.1in
  }%
}
\makeatother
\AddToHook{cmd/maketitle/after}{\thispagestyle{empty}}
\date{}
\let\siconbibliography\thebibliography
\let\endsiconbibliography\endthebibliography
\usepackage[numbers,square,comma]{natbib}
\renewenvironment{thebibliography}[1]{%
  \setcounter{NAT@ctr}{0}%
  \siconbibliography{#1}%
}{\endsiconbibliography}
\usepackage{amssymb,mathtools,stmaryrd}
\usepackage{algpseudocode}
\SetSymbolFont{stmry}{bold}{U}{stmry}{m}{n}
\usepackage{graphicx}
\usepackage{array,xcolor,framed,tikz}
\usetikzlibrary{arrows.meta,calc}
\usepackage{pgfplots}
\usepgfplotslibrary{groupplots}
\pgfplotsset{compat=1.18}
\usepackage{fix-cm}
\newcommand{\figurefont}{\rmfamily\fontsize{8.5bp}{10bp}\selectfont}
\newcommand{\plotfont}{\figurefont} 
\crefname{section}{Section}{Sections}
\crefname{subsection}{Section}{Sections}
\Crefname{subsection}{Section}{Sections}
\crefname{subsubsection}{Section}{Sections}
\Crefname{subsubsection}{Section}{Sections}
\usepackage[acronym]{glossaries}
\glsdisablehyper
\newacronym{sci}{SCI}{sequential contact insertion}
\newacronym{scp}{SCP}{sequential convex programming} 
\providecommand{\norm}[1]{\lVert #1 \rVert}
\providecommand{\inner}[1]{\langle #1 \rangle}
\providecommand{\abs}[1]{\left\lvert #1 \right\rvert}
\providecommand{\dd}{\mathop{}\!\mathrm{d}}
\providecommand{\jmp}[1]{\llbracket #1 \rrbracket} 
\crefname{subsection}{Section}{Sections}
\Crefname{subsection}{Section}{Sections}
\crefname{subsubsection}{Section}{Sections}
\Crefname{subsubsection}{Section}{Sections}
\crefname{appendix}{Appendix}{Appendices}
\Crefname{appendix}{Appendix}{Appendices}
\crefname{subappendix}{Section}{Sections}
\Crefname{subappendix}{Section}{Sections}
\crefname{subsubappendix}{Section}{Sections}
\Crefname{subsubappendix}{Section}{Sections}
\newif\ifproofInMain
\newif\ifproofInProofs
\proofInMainfalse
\proofInProofstrue

\newcommand{\PaperTitle}{A complete analytical characterization of optimal collision-free motion for two robots}
\newcommand{\PaperAuthors}{Jan Kamm\thanks{\PaperAffiliation} \and
  Antonio Terpin\footnotemark[1] \and
  Raffaello~D'Andrea\footnotemark[1]}
\newcommand{\PaperAffiliation}{Institute for Dynamic Systems and Control,
  ETH Z\"urich, Z\"urich, Switzerland.\\
  Jan Kamm and Antonio Terpin contributed equally to this work.\\
  Corresponding author: Jan Kamm
  (\href{mailto:jakamm@ethz.ch}{jakamm@ethz.ch}).\\
  ORCID: Jan Kamm,
  \href{https://orcid.org/0009-0005-6289-4198}{0009-0005-6289-4198};\\
  Antonio Terpin,
  \href{https://orcid.org/0000-0002-8456-6786}{0000-0002-8456-6786};\\
  Raffaello~D'Andrea,
  \href{https://orcid.org/0000-0001-5287-7849}{0000-0001-5287-7849}}
\newcommand{\PaperAbstract}{Collision-free motion planning for multi-agent robotic
systems remains an active research area because of the complexity of coordinating
robot trajectories, with applications ranging from coordinated drone shows
to autonomous operations on land, in the air, and on water. 
Importantly, the analytical
difficulty of nonconvex collision-avoidance constraints has so far driven the development
of exclusively numerical methods. 
We derive a complete analytical characterization of the optimal
trajectories for two robots with zero initial and final velocities,
under double-integrator dynamics and a squared-acceleration objective.
We prove that these trajectories belong to eight distinct contact
families and determine which family applies for each configuration of initial
and final positions. We analyze the geometrical properties of the optimal trajectories and provide a numerical algorithm for computing them.
These insights have already been adopted to improve the performance of numerical methods for multi-agent motion planning
by several orders of magnitude.
}

\newcommand{\PaperKeywords}{two-robot motion planning, trajectory optimization,
  optimal control, collision avoidance, state constraints}
\newcommand{\PaperMSC}{Primary 49K15; Secondary 49M05, 49M25, 70Q05} 
\title{\PaperTitle}
\expandafter\author\expandafter{\PaperAuthors}
\hypersetup{pdftitle={\PaperTitle},
  pdfauthor={Jan Kamm, Antonio Terpin, Raffaello D'Andrea}} 
\begin{document}
\maketitle

\begin{abstract}
\PaperAbstract
\end{abstract}



\section{Introduction}
Collision-free motion planning for multi-agent robotic systems remains an active
research area because of the complexity of coordinating robot trajectories,
with applications ranging from coordinated drone shows to autonomous operations
on land \citep{wurman_coordinating_2008}, in the air
\citep{augugliaro_generation_2012}, and on water \citep{kamm_floating_2026}.
A variety of numerical methods exist to solve this problem, from sampling and
graph search
\citep{karaman_sampling_2011,shome_drrt_2020,sharon_conflict_2015,wagner_subdimensional_2015}
to joint assignment and planning \citep{turpin_capt_2014} and hybrid search and
trajectory optimization \citep{honig_trajectory_2018}.
Trajectory optimizers use mixed-integer formulations \citep{earl_iterative_2005},
spline parameterizations \citep{vanloock_spline_2015,wang_geometric_2022},
\gls{scp} \citep{augugliaro_generation_2012,schulman_trajectories_2013,bonalli_gusto_2019,elango_continuous_2025}, learning-based methods
\citep{chen_decentralized_2017,grontas_pinet_2026}, distributed optimization
\citep{vanparys_distributed_2017,terpin_feedback_2022}, and separation principles \citep{terpin_dynamic_2024,pezzetti2026a}.

We consider collision-free motion planning under double-integrator dynamics,
with fixed final time and endpoint positions.
The squared-acceleration objective measures control effort and is equivalent to
squared-thrust minimization \citep{augugliaro_generation_2012}.
This objective is commonly used for multi-agent assignment and collision avoidance
\citep{beaver_assignment_2020,beaver_energy_2021}, navigation among polygonal
obstacles \citep{beaver_graph_2023}, and, alongside other cost terms,
distributed quadrotor planning
\citep{luis_trajectory_2019}.
Various drone and floating-robot experiments have demonstrated 
that simplified planning models generate
reference trajectories amenable to tracking by low-level controllers
\citep{augugliaro_generation_2012,luis_trajectory_2019,kamm_floating_2026}.
However, these approaches compute trajectories numerically because nonconvex
collision-avoidance constraints make an analytical solution difficult. Even for
two robots, existing theory provides no complete analytical characterization of minimum-effort
transitions between prescribed endpoint states over a fixed time horizon under
double-integrator dynamics.

Importantly, the two-robot case isolates the pairwise interaction that underlies
fleet planning, and an analytical characterization would identify the fundamental
properties of optimal collision-free motion, offering new possibilities for
efficient planning.
Analytical characterizations exist in other settings.
\citet{hu_optimal_2003} construct exact optimal motions for two planar
agents through a geometric analysis of the separation constraint.
They minimize squared velocity and leave endpoint velocities free.
\citet{kirkpatrick_characterizing_2016,kirkpatrick_minimum_2025} give a complete
geometric characterization for two robots minimizing total path length, with
paths consisting of straight segments and circular arcs.
For the squared-acceleration objective, \citet{beaver_assignment_2020}
construct trajectories by successively inserting constrained arcs and solving
for their joining times. Their method optimizes each agent against prescribed
higher-priority trajectories and restricts sustained contact to constant relative
speed. A complete characterization of the jointly optimal two-robot motion
requires accounting for variable relative speed and determining which combinations
of isolated contacts and sustained contact are globally optimal.

\begingroup
\definecolor{contributionsPanelBlue}{RGB}{240,245,250}
\colorlet{shadecolor}{contributionsPanelBlue}
\setlength{\FrameSep}{8pt}
\brokenpenalty=10000
\begin{shaded*}
\ifdefined\internallinenumbers
  \internallinenumbers
  \addtolength{\linenumbersep}{\FrameSep}
\fi
    \noindent\textbf{Contributions.}
    In this paper, we provide a complete analytical characterization for the
    optimal trajectories of two robots with zero initial and final velocities under double-integrator dynamics and a squared-acceleration objective.
        We prove that the contact
        set consists of finitely many isolated points and boundary arcs,
        with cubic motion between contacts, and we derive an exact cost identity that gives
        a sufficient condition for global optimality
        (\cref{sec:necessary,sec:certificate}).
        Then, we prove that every optimal relative trajectory lies in a plane, classify the minimizers into
        eight contact families, and show that their admissible construction
        images partition the endpoint space
        (\cref{sec:zero-endpoint-velocities}).
        We give a certified numerical procedure that recovers the selected
        trajectory and its cost to any prescribed positive accuracy for
        algebraic endpoint data (\cref{sec:numerical-recovery}).
        Our two-robot characterization underlies the contact-based fleet planner
        in~\citep{kamm_contact_2026}, which incorporates acceleration
        constraints and computes collision-free trajectories for one million
        robots in minutes and for thousands in a fraction of a second.
\end{shaded*}
\endgroup 
\section{Results}
\label{sec:results}

We consider two robots with zero initial and final velocities, strict initial
separation and terminal contact allowed.
Their relative position $r(t)$ must remain
outside the interior of an exclusion ball whose radius $R$ is the
required separation. We call an encounter at the minimum separation a
\emph{contact}. The relative trajectory may touch its boundary at isolated points
or follow it along an arc, representing sustained contact
(\cref{fig:contact-catalogue}(d--g)).
Every global optimum belongs to one of eight contact families
(\cref{fig:contact-catalogue}(a--c)).
When the endpoint segment avoids the disk's interior, the
free cubic is optimal (F). Otherwise, the relative trajectory joins distinct
cubic pieces at one, two, or three isolated contacts (P, PP, PPP), or follows a
single boundary arc (A), possibly with one isolated contact before
(PA), after (AP), or on each side (PAP).
\Cref{fig:contact-catalogue}(a--c) illustrates numerically how the optimal contact pattern depends on the
robots' initial and final relative positions. Every globally minimizing relative
trajectory lies
in a plane (\cref{prop:planar-reduction}). In that plane, the initial and
final relative positions are $(d_0,0)$ and $(r_{T,1},r_{T,2})$, respectively,
in units of $R$.
\begin{figure}[tp]
  \centering
  \scalebox{0.88}{\begin{minipage}{\linewidth}\centering
\definecolor{phaseC0}{HTML}{BCCEE1}
\definecolor{phaseP}{HTML}{EBCEC2}
\definecolor{phasePP}{HTML}{D6C2D2}
\definecolor{phasePPP}{HTML}{B5D6CE}
\definecolor{phaseA}{HTML}{D5D5D5}
\definecolor{phasePA}{HTML}{EDE7BC}
\definecolor{phaseAP}{HTML}{F9C8CE}
\definecolor{phasePAP}{HTML}{D0E4AC}
\definecolor{phaseunresolved}{HTML}{EEEEEE}
\definecolor{overviewfree}{HTML}{4478AB}
\definecolor{overviewP}{HTML}{C87855}
\definecolor{overviewPP}{HTML}{8C5582}
\definecolor{overviewPPP}{HTML}{328C78}
\definecolor{overviewleft}{HTML}{CDBB44}
\definecolor{overviewright}{HTML}{EF6677}
\usetikzlibrary{patterns.meta}
\tikzset{
  phase texture C0/.style={pattern={Lines[angle=45,distance=6pt,line width=.3pt]},pattern color=black!28},
  phase texture P/.style={pattern={Lines[angle=-45,distance=6pt,line width=.3pt]},pattern color=black!28},
  phase texture PP/.style={pattern={Lines[angle=0,distance=6pt,line width=.3pt]},pattern color=black!28},
  phase texture PPP/.style={pattern={Lines[angle=90,distance=6pt,line width=.3pt]},pattern color=black!28},
  phase texture A/.style={},
  phase texture PA/.style={pattern={Hatch[angle=45,distance=7pt,line width=.3pt]},pattern color=black!28},
  phase texture AP/.style={pattern={Hatch[angle=0,distance=7pt,line width=.3pt]},pattern color=black!28},
  phase texture PAP/.style={pattern={Dots[distance=5pt,radius=.45pt]},pattern color=black!32},
  phase texture unresolved/.style={},
}  
\newcommand{\slicetag}[1]{{\plotfont(#1)}}
\newcommand{\swatch}[2]{\tikz[baseline={([yshift=-.3em]current bounding box.center)}]{\fill[phase#1](0,0) rectangle (.30,.26);\ifx\relax#2\relax\else\draw[overview#2,line width=.5pt](0,0) rectangle (.30,.26);\fi}\,}
\begingroup
\figurefont
\raisebox{0sp}[23838550sp][0sp]{\includegraphics[width=22876058sp,height=23838550sp]{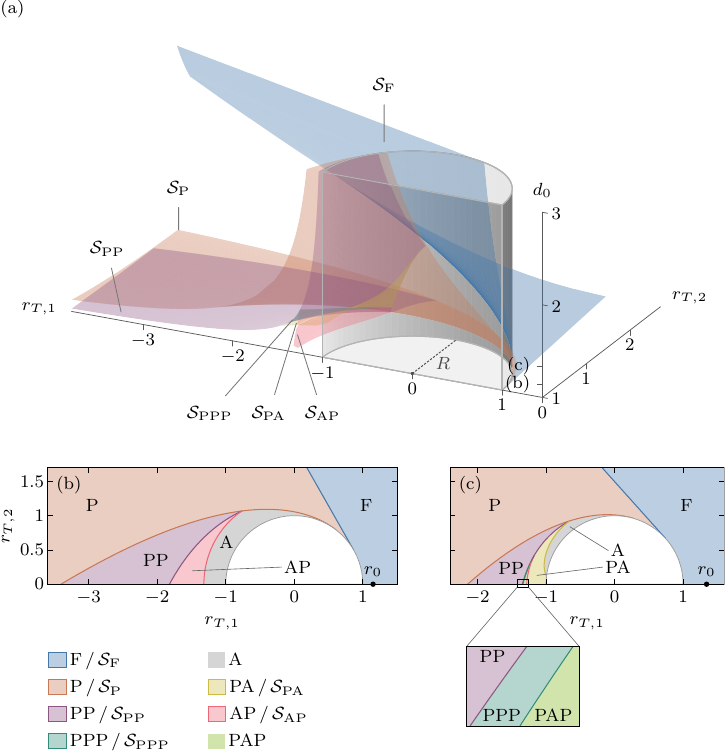}}
\endgroup{} 
  \end{minipage}}\par
  \vspace{2pt}
  \scalebox{0.9}{\begin{minipage}{\linewidth}\centering
\begingroup
\definecolor{freeMotionBlue}{RGB}{68,120,171}
\definecolor{contactRed}{RGB}{168,66,40}
\pgfplotsset{
  title/.code={},
  exclusion disk/.style={draw=black!30,fill=black!5,line width=.45pt,
    domain=0:360,samples=101},
  disk center/.style={only marks,mark=*,mark size=.7pt,
    mark options={draw=none,fill=black!50}},
  free piece/.style={draw=freeMotionBlue,line width=1.1pt,no marks,
    line cap=round,line join=round},
  boundary arc/.style={draw=contactRed,line width=1.8pt,no marks,
    line cap=round,line join=round,
    preaction={draw=black,line width=2.1pt}},
  reaction mark/.style={only marks,mark=*,mark size=1.6pt,
    mark options={draw=none,fill=contactRed}},
  start mark/.style={only marks,mark=*,mark size=1.4pt,
    mark options={draw=freeMotionBlue,fill=freeMotionBlue,line width=.75pt}},
  finish mark/.style={only marks,mark=*,mark size=1.4pt,
    mark options={draw=freeMotionBlue,fill=white,line width=.75pt}},
}
\figurefont
\raisebox{0sp}[6979822sp][0sp]{\includegraphics[width=19793795sp,height=6979822sp]{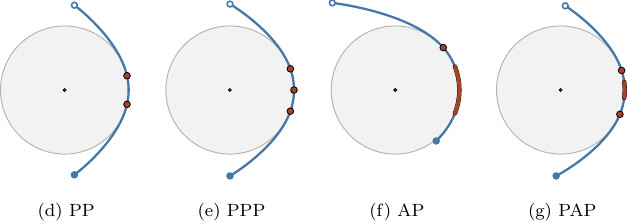}}
\endgroup{} 
  \end{minipage}}\par
  \vspace{-6pt}
  \caption{Two-robot contact structure in relative coordinates.
    (a--c) Numerical contact-family map for $R=T=1$: transition surfaces
    and slices at $d_0=1.15$ and $d_0=1.343$; the inset magnifies the
    narrow PPP region. The cylinder excludes $\norm{r_T}<1$.
    (d--g) Representative trajectories with $R=1$.
    Blue denotes free motion; red dots and segments mark isolated contacts
    and boundary arcs. Filled and open dots mark initial and final positions.}
  \label{fig:contact-catalogue}
  \vspace{-6pt}
\end{figure}
In the following we present the key mathematical results that
make this characterization precise.

\subsection{Problem formulation}
\label{sec:problem_setup}

Fix a time horizon $T>0$. For integers $d\geq1$ and $k\geq0$,
$L^2((0,T);\mathbb{R}^d)$ is the space of square-integrable functions,
and $H^k((0,T);\mathbb{R}^d)$ consists of functions whose weak
derivatives through order $k$ are square-integrable
\citep[Chapters~2--3]{adams2003}.
We write $C^k([0,T];\mathbb{R}^d)$ for functions with $k$ continuous
derivatives.
For $z,w\in\mathbb{R}^d$ and $f\in L^2((0,T);\mathbb{R}^d)$, set
$\inner{ z,w}\coloneqq\sum_{\ell=1}^d z_\ell w_\ell$,
$\norm{z}\coloneqq\inner{ z,z}^{1/2}$, and
$\norm{f}_{L^2((0,T);\mathbb{R}^d)}
\coloneqq\left(\int_0^T\norm{f(t)}^2\dd t\right)^{1/2}$.
We omit the norm subscript when the argument makes its meaning
unambiguous. We write $\abs{A}$ for the cardinality of a set $A$
and $\abs{a}$ for the absolute value of a scalar $a$.
The vectors $e_1,\ldots,e_d$ denote the orthonormal coordinate vectors
of the chosen spatial frame.
Square brackets enclose function or trajectory arguments; parentheses
indicate evaluation at a time or a finite-dimensional argument.
Write $D_t^k$ for the $k$th distributional time derivative, including
$D_t^1$ for the first. Ordinary time derivatives are denoted by
$\dot{r}$, $\ddot{r}$, $r^{(3)}$, and $r^{(4)}$, with the same
convention for other trajectories. For Sobolev trajectories, these
symbols denote the corresponding weak derivatives. Dots on normalized
trajectories refer to their displayed time variable.
Time differentiation acts componentwise on vector-valued functions.
We write $\mathcal{M}_+([0,T])$ for finite nonnegative Radon measures
and $\delta_s$ for the Dirac measure with unit mass at $s$.
The support $\operatorname{supp}(\mu)$ consists of points whose every
neighborhood in $[0,T]$ has positive $\mu$-measure.
For a function with finite one-sided limits, write
$f(t^-)\coloneqq\lim_{s\uparrow t}f(s)$,
$f(t^+)\coloneqq\lim_{s\downarrow t}f(s)$, and
$\jmp{f}_t\coloneqq f(t^+)-f(t^-)$.
The minus and plus signs indicate limits from the left and right,
respectively; $f(0^+)$ and $f(T^-)$ are the endpoint limits from inside
the time interval. A right-continuous representative takes the value
$f(t^+)$ at each interior time.
In measure integrals, unmarked limits include endpoints, while a
lower limit $a^+$ excludes $a$ and an upper limit $b^-$ excludes $b$.
For example, $\int_{0^+}^{t}f(s)\dd\mu(s)$ integrates over
$0<s\leq t$, whereas $\int_{0^+}^{T^-}f(s)\dd\mu(s)$ integrates
over $0<s<T$.

\subsubsection{Two-robot model}
Let $r,b\in H^2((0,T);\mathbb R^d)$, $d\geq2$, denote the position of
the first robot relative to the second and their barycenter, respectively.
Both trajectories have prescribed endpoint positions $r_0,r_T$ and
$b_0,b_T$, and zero initial and final velocities.
The endpoint data satisfy $\norm{r_0}>R$ and $\norm{r_T}\geq R$,
where $R>0$ is the minimum separation.
The cost is
$2\norm{\ddot b}^2+\frac12\norm{\ddot r}^2$.
Since separation constrains only $r$, the barycentric and relative
trajectories minimize their costs independently.
The unique optimal barycenter is the cubic Hermite interpolant
\citep[Section~2.5]{egerstedt_martin_2010}
$b^{\mathrm H}(t)=b_0+(b_T-b_0)\hat t^2(3-2\hat t)$,
where $\hat t=t/T$.
The relative trajectory solves
\begin{equation}
  \min_r\frac12\norm{\ddot r}^2
  \;\text{s.t.}\; r(0)=r_0, r(T)=r_T,
  \dot r(0)=\dot r(T)=0, \norm{r(t)}\geq R, t\in[0,T].
  \label{eq:reduced}
\end{equation}
Let $\mathcal M^*$ denote its set of global minimizers.

 \begin{remark}[Nondimensionalization]\label{rem:nondimensionalization}
For $\hat t=t/T\in[0,1]$, define $\hat r(\hat t)\coloneqq r(T\hat t)/R$.
The inverse $r(t)=R\hat r(t/T)$ identifies feasible trajectories
with those of~\eqref{eq:reduced} at $T=R=1$, with endpoint positions
$r_0/R$, $r_T/R$ and zero initial and final velocities. Moreover,
$\norm{\ddot r}_{L^2((0,T);\mathbb R^d)}^2
=R^2T^{-3}\norm{\ddot{\hat r}}_{L^2((0,1);\mathbb R^d)}^2$.
Thus minimizers correspond bijectively. Hats denote normalization;
we omit them when $T=R=1$. \end{remark}  

\begin{remark}[Existence and nonuniqueness]
The direct method \citep[Theorem~2.3]{rindler2018} guarantees a global
minimizer of~\eqref{eq:reduced}.
This solution may not be unique.
In fact, uniqueness fails for opposite endpoint vectors, since symmetry gives
equally costly optimal trajectories passing on opposite sides of the
separation ball.
\end{remark} 
\subsection{Optimal trajectory characterization}
\label{sec:two-robots}

To characterize the structure of the optimal trajectories,
we first derive optimality conditions
(\cref{prop:relative-multipliers}) and construct a global certificate
(\cref{sec:certificate}).
Using the optimality conditions, we characterize the active set, namely
the times at which the robots attain the minimum separation, as finitely
many isolated contacts and boundary arcs (\cref{prop:contact-structure}).
We then prove that every optimal relative trajectory lies in a plane
and establish side selection to complete the catalogue
of eight contact families and
determine their regions in endpoint space (\cref{sec:zero-endpoint-velocities}).
We conclude with a numerical procedure to recover the selected optimal
trajectory and its cost (\cref{sec:numerical-recovery}).

\subsubsection{Contact laws}
\label{sec:necessary}

Write
$\mathcal A=\{t\in[0,T]:\norm{r(t)}=R\}$ for the active set and
$g(t)=R^2-\norm{r(t)}^2$ for the clearance.
The connected components of $\mathcal A$ are the contacts.
An isolated point is an \emph{isolated contact}; a contact of positive
length is a \emph{boundary arc}, necessarily a closed interval $[a,b]$
with $a<b$.
A measure represents both isolated contacts and contact maintained over
an interval.

 \begin{proposition}[Karush--Kuhn--Tucker conditions for the relative problem]\label{prop:relative-multipliers}
If $\norm{r_T}>R$, every local minimizer $r^*$ of~\eqref{eq:reduced}
in $H^2((0,T);\mathbb R^d)$ admits $\mu\in\mathcal M_+([0,T])$ such that
$\int_0^T\inner{\ddot r^*(t),\ddot\eta(t)}\dd t
=\int_0^T\inner{r^*(t),\eta(t)}\dd\mu(t)$
for every $\eta\in H^2((0,T);\mathbb R^d)$ with zero endpoint values
and first derivatives. Moreover,
$\int_0^T\bigl(R^2-\norm{r^*(t)}^2\bigr)\dd\mu(t)=0$ and
$\operatorname{supp}(\mu)\subseteq\mathcal A$. \end{proposition}  

The multiplier $\mu$ describes how the separation constraint changes the
free cubic motion: at an isolated contact, positive mass produces a jump
in jerk, whereas zero mass lets the same cubic continue through the contact.

If $\norm{r_T}>R$, \cref{prop:relative-multipliers} gives every local minimizer a
multiplier measure $\mu$ satisfying
\begin{equation}
  D_t^4r=r\,\mu,
  \quad \mu\geq0,
  \quad \operatorname{supp}(\mu)\subseteq\mathcal A,
  \label{eq:euler-lagrange}
\end{equation}
in distributions on $(0,T)$. On boundary arcs, the density of $\mu$
determines the rate of change of jerk.
We call a feasible trajectory $r$ \emph{stationary} if some finite
nonnegative Radon measure $\mu$ satisfies~\eqref{eq:euler-lagrange}; the
conclusions below apply to every such pair $(r,\mu)$.

 \begin{proposition}[Contact structure]\label{prop:contact-structure}
The active set has finitely many connected components, each an isolated
contact or a boundary arc:
$\mathcal{A}=\{t_i\}_{i=1}^{m}\cup\bigcup_{\ell=1}^{k}[a_\ell,b_\ell]$.
The multiplier has no singular-continuous part. Its atoms lie at
isolated contacts and arc endpoints; its remaining density is
nonnegative, Lipschitz continuous on each arc interior, and zero
almost everywhere outside the arcs. \end{proposition}  

A stationary trajectory is thus a finite chain of free cubics, isolated
contacts, and boundary arcs. 
\subsubsection{Global optimality certificate}
\label{sec:certificate}

The conditions in \cref{sec:necessary} characterize stationary trajectories,
which may be only locally optimal. To certify global optimality, we compare
a stationary trajectory $r$ with every feasible trajectory $q$. We use a
nonnegative multiplier measure $\mu$ for the separation constraint along $r$,
satisfying~\eqref{eq:euler-lagrange}, to derive an exact identity for their
cost difference.

\paragraph{The free cubic and its correction.}
Let $\bar r$ be the unique unconstrained minimizer of
$\frac12\norm{\ddot r}^2$ with the
prescribed endpoint states. It is the cubic Hermite interpolant of
the relative endpoint data.
Every trajectory with these endpoint states has the form $r=\bar r+h$,
where $h$ belongs to the space $\mathcal H_T$ of functions in
$H^2((0,T);\mathbb R^d)$ with
${h(0)=h(T)=\dot{h}(0)=\dot{h}(T)=0}$.
Equip this space with the inner product
$\inner{ h_1,h_2}_{\mathcal H_T}\coloneqq\int_0^T
\inner{\ddot{h}_1(t),\ddot{h}_2(t)}\dd t$.
The correction adds its own acceleration cost to that of the free cubic
(\cref{app:spline-identities}):
$\frac12\norm{\ddot{\bar r}+\ddot h}^2
=\frac12\norm{\ddot{\bar r}}^2+\frac12\norm{h}_{\mathcal H_T}^2$.

\paragraph{From multipliers to trajectories.}
Restrict the measure $\mu$ to $(0,T)$.
The stationarity equation~\eqref{eq:euler-lagrange} gives
$D_t^4h=r\,\mu$ for the correction $h=r-\bar r$, with zero endpoint
position and velocity. The scalar Green kernel $G_T(t,s)$ describes
the correction at time $t$ produced by a unit jump in jerk at time $s$
under these endpoint conditions. Rescaling the unit-interval formula
\citep[Equation~(3.3)]{ma_yan_wei_2020} by
$G_T(t,s)=T^3G_1(t/T,s/T)$ gives
\refstepcounter{equation}\label{eq:relaxation-kernel}%
$G_T(t,s)=\frac{t^2(T-s)^2}{6T^3}(3Ts-Tt-2ts)$~\textup{(\theequation)}
for $0\leq t\leq s\leq T$, with $G_T(t,s)=G_T(s,t)$ for $t>s$.
Adding the corrections from all
contacts gives (\cref{app:proof-contact-certificate})
\refstepcounter{equation}\label{eq:reaction-green-reconstruction}%
$r(t)=\bar r(t)+\int_{0^+}^{T^-}G_T(t,s)\,r(s)\dd\mu(s)$~\textup{(\theequation)}
for $t\in[0,T]$.

\paragraph{Comparing costs.}
Define
$Q_\mu[h]\coloneqq\norm{h}_{\mathcal H_T}^2
-\int_{0^+}^{T^-}\norm{h(t)}^2\dd\mu(t)$ for $h\in\mathcal H_T$.
We can now compare $r$ with any trajectory $q=r+h$ with the same endpoint
states through the exact identity
\begin{equation}
  \frac12\bigl(\norm{\ddot q}^2-\norm{\ddot r}^2\bigr)=\frac12Q_\mu[h]
  +\frac12\int_{0^+}^{T^-}(\norm{q(t)}^2-R^2)\dd\mu(t).
  \label{eq:cost-identity}
\end{equation}
For every feasible $q$, the separation constraint makes the last integral
nonnegative. Thus, $Q_\mu\geq0$ on $\mathcal H_T$ certifies global
optimality of $r$. The following theorem gives a sufficient condition for
this nonnegativity in terms of projections of $r$ and $\bar r$ on
$\operatorname{supp}(\mu)$.

 \begin{theorem}[A global optimality certificate]\label{thm:contact-certificate}
Let $r$ be stationary with multiplier $\mu$, and let $e\in\mathbb R^d$ be a
unit vector and $\delta\geq0$ such that
$\bar z(t)\coloneqq\inner{ e,\bar r(t)}\geq\delta$ and
$z(t)\coloneqq\inner{ e,r(t)}>0$ for
$t\in\operatorname{supp}(\mu)\cap(0,T)$.
Then, with $\kappa_{\mathrm{cert}}\coloneqq\min\{1,\delta/R\}$,
$Q_\mu[h]\geq\kappa_{\mathrm{cert}}\norm{h}_{\mathcal H_T}^2$ and
$\frac12\bigl(\norm{\ddot q}^2-\norm{\ddot r}^2\bigr)
\geq\frac{\kappa_{\mathrm{cert}}}{2}\norm{q-r}_{\mathcal H_T}^2$
for every feasible $q$.
If $\delta>0$, $r$ is the unique minimizer. \end{theorem}  

Geometrically, the contact positions at times in $\operatorname{supp}(\mu)$ lie in one open hemisphere
and the free motion lies in the corresponding closed half-space at those
times; this alignment certifies optimality even against trajectories
passing around the opposite side of the ball. 
\subsubsection{The contact catalogue}
\label{sec:zero-endpoint-velocities}
\label{sec:canonical-constructions}

With zero initial and final velocities the free cubic runs along the segment joining the
endpoints, so its projection on the normal to that segment is constant:
the endpoint data alone fix a direction for
\cref{thm:contact-certificate}.
Scale time and separation to $T=R=1$ as in \cref{sec:problem_setup}.
The endpoint data satisfy $\norm{r_0}>1$ and $\norm{r_T}\geq1$.

\paragraph{Normalized endpoint data.}
Choose orthonormal vectors $e_1,e_2\in\mathbb R^d$ with $r_0=d_0e_1$ and
$r_T=r_{T,1}e_1+r_{T,2}e_2$, $r_{T,2}\geq0$; if $r_T$ is parallel to $r_0$,
$e_2$ is any unit vector orthogonal to $e_1$. Then the \emph{datum}
$\omega\coloneqq(d_0,r_{T,1},r_{T,2})$ lies in
$\mathcal D\coloneqq\{(d_0,r_{T,1},r_{T,2})\in\mathbb R^3:
d_0>1,\ r_{T,1}^2+r_{T,2}^2\geq1,\ r_{T,2}\geq0\}$.
The \emph{symmetry axis} $r_{T,2}=0$ and the \emph{terminal wall}
$\norm{r_T}=1$ are the faces of $\mathcal D$.
Identify the \emph{chosen plane} $\operatorname{span}\{e_1,e_2\}$ with
$\mathbb R^2$, and write $a\times b\coloneqq a_1b_2-a_2b_1$ for planar
vectors.

The free cubic $\bar r(t)=r_0+(r_T-r_0)t^2(3-2t)$, $0\leq t\leq1$,
traverses the endpoint segment. Call the data \emph{blocked} when that segment
meets the open unit disk, that is, when the free cubic is infeasible.
Unblocked data are
settled by \cref{thm:contact-certificate} at $\mu=0$: the free cubic is their
unique minimizer. Call data on the symmetry axis \emph{symmetric} and all
other data \emph{off-axis}. For off-axis data the oriented unit normal to the
line through the endpoints is
$e\coloneqq(r_{T,2},d_0-r_{T,1})/\norm{r_T-r_0}$.
With $\delta\coloneqq\inner{ e,r_0}=\inner{ e,r_T}
=d_0r_{T,2}/\norm{r_T-r_0}$, we have
$\inner{ e,\bar r(t)}=\delta$ throughout. For symmetric data set
$e=(0,1)$ and $\delta=0$. With this $e$, the hypotheses of
\cref{thm:contact-certificate} reduce to $\inner{ e,r(t)}>0$ on
$\operatorname{supp}(\mu)$. The \emph{near side} is the half-plane
$\inner{ e,\cdot}\geq\delta$ in the chosen plane; for off-axis data, it
is the side of the line through the endpoints away from the center of the
disk. A trajectory in the chosen plane is \emph{upper} if
$\inner{ e,r(t)}>0$ for $0<t<1$.

 \begin{proposition}[Planar reduction]\label{prop:planar-reduction}
For every $r\in\mathcal M^*$, the image $r([0,1])$ lies in a
two-dimensional linear subspace of $\mathbb R^d$. \end{proposition}  

Intuitively, \cref{prop:planar-reduction} shows that additional spatial directions cannot reduce the minimum effort.

 \begin{proposition}[Side selection]\label{prop:side-selection}
For off-axis data, $\mathcal M^*=\{r_*\}$ and $r_*([0,1])$ is contained
in the near side. For blocked symmetric data, $\mathcal M^*$ is the orbit of its unique
upper element under the orthogonal maps of $\mathbb R^d$ that fix $r_0$
and $r_T$. \end{proposition}  

For off-axis data with $\norm{r_T}>1$, the two ways around the disk each
carry a local minimizer among trajectories in the chosen plane
(\cref{lem:homotopy-class-minimizers}), and \cref{prop:side-selection} puts the global
one on the near side. The \emph{selected minimizer} is thus
well defined: the unique minimizer for
off-axis or unblocked data, and the upper element of $\mathcal M^*$ for
blocked symmetric data.

\paragraph{The contact catalogue.}
With the side fixed, the laws of \cref{sec:necessary} limit how the
selected minimizer can touch the disk.
Every global minimizer has a unique finite interior multiplier $\mu$
satisfying~\eqref{eq:euler-lagrange} (\cref{prop:canonical-reaction}).
The labels in \cref{fig:contact-catalogue}(a--c) record its contact pattern.
Write P for an isolated contact at $t$ with positive mass $\mu(\{t\})>0$,
A for a maximal boundary arc together with the atoms at its endpoints,
and F for $\mu=0$; concatenation records temporal order, coincident atoms
combine, and a zero-mass touch, like a prescribed endpoint contact,
contributes no letter.

 \begin{theorem}[Canonical contact catalogue]\label{thm:canonical-catalogue}
For every datum in $\mathcal D$, the contact pattern of the
selected minimizer is one of
$\mathrm F$, $\mathrm P$, $\mathrm{PP}$, $\mathrm{PPP}$,
$\mathrm A$, $\mathrm{PA}$, $\mathrm{AP}$, and $\mathrm{PAP}$. \end{theorem}  

The catalogue reduces the search for the optimum to eight
finite-dimensional constructions.
\Cref{fig:contact-catalogue}(d--g) shows representatives of four of the eight families.

\paragraph{Atomic families.}
For blocked data, choose $m\in\{1,2,3\}$, ordered times
$0<t_1<\cdots<t_m<1$, and jerk jumps $\Delta_i\in\mathbb R^2$.
With atoms at these times, the kernel
response~\eqref{eq:reaction-green-reconstruction} takes the form
$r(t)=\bar r(t)+\sum_{i=1}^{m}G_1(t,t_i)\Delta_i$.
With $n_i=r(t_i)$, $c_i=(t_i,\Delta_i)$, and $c=(c_1,\ldots,c_m)$,
the contact conditions form the system
\begin{equation}
  F(c)\coloneqq
  \bigl(\norm{n_i}^2-1,\;n_i\times\Delta_i,\;
  \inner{ n_i,\dot r(t_i)}\bigr)_{i=1}^{m}=0.
  \label{eq:atomic-contact-system}
\end{equation}
These are $3m$ equations in $3m$ unknowns for P, PP, and PPP;
the atom masses must also satisfy $m_i\coloneqq\inner{ n_i,\Delta_i}>0$.
Each real solution of~\eqref{eq:atomic-contact-system} defines a candidate
trajectory. It satisfies the separation constraint exactly when the polynomial
$\norm{r(t)}^2-1$ is nonnegative on every cubic piece; the tests are given
in~\cref{app:canonical-atomic}.

\paragraph{Arc families.}
On a boundary arc, the robots remain at the minimum separation. Their
relative position is described by a single angle, whose rate of change
satisfies the differential equation in~\cref{prop:boundary-dynamics}.
At either end a free cubic reaches the endpoint with zero velocity
directly, giving A, or passes through one isolated contact first, giving PA,
AP, or PAP according to the side.
\Cref{app:proof-endpoint-regime-map} evaluates the
arc system and derives both completions.
A construction is \emph{admissible} if its endpoint triple lies in
$\mathcal D$, its atom masses are positive,
every free piece clears the disk, $\inner{ e,\cdot}>0$ on the support
of its multiplier, and any arc has positive speed and nonnegative density. 
\paragraph{Regime regions.}
Let $\mathfrak L=\{\mathrm F,\mathrm P,\mathrm{PP},\mathrm{PPP},
\mathrm A,\mathrm{PA},\mathrm{AP},\mathrm{PAP}\}$. For $\sigma\in\mathfrak L$,
let $\mathcal P_\sigma$ be the parameters of the admissible constructions of
\cref{sec:canonical-constructions} whose contact pattern has label $\sigma$,
retaining all admitted roots, and let $\mathcal E_\sigma(p)$ be the normalized
endpoint triple of the assembled trajectory, using the map
$\mathcal E_{\mathrm{can}}$ in \cref{app:canonical-completion} when rotation is needed. Set
$\mathcal R_\sigma\coloneqq\mathcal E_\sigma(\mathcal P_\sigma)\subseteq\mathcal D$.
Thus $\mathcal R_{\mathrm F}$ is the set of unblocked data.

 \begin{theorem}[Regime map partition]\label{thm:endpoint-regime-map}
The regions $\mathcal R_\sigma$, $\sigma\in\mathfrak L$, are pairwise
disjoint and cover $\mathcal D$.
For $\omega\in\mathcal R_\sigma$, every admissible construction mapping to $\omega$
gives the selected minimizer, with label $\sigma$, and every
selected minimizer arises this way. \end{theorem}  

\Cref{fig:contact-catalogue}(a--c) illustrates these regions numerically.
The procedure in~\cref{sec:numerical-recovery} computes the selected optimal
trajectory to a prescribed accuracy.

\subsubsection{Numerical recovery}
\label{sec:numerical-recovery}

\Cref{alg:canonical-recovery} first checks the free cubic, then the
families P, PP, and PPP. If none applies, it searches A, PA, AP, and PAP
by narrowing intervals for their contact parameters.
The normalized endpoint data $\omega\in\mathcal D$
(\cref{rem:nondimensionalization}) are algebraic, and the requested
accuracy $\varepsilon>0$ is rational.
Write $J[q]=\frac12\norm{\ddot q}^2$ for the cost.
The output is a trial trajectory $\widetilde r$, a bound $\rho$ on its
position, velocity, and acceleration errors relative to the selected
minimizer at every time, and an
interval $I$ containing the optimal cost. Both $\rho$ and the width of
$I$ must be at most $\varepsilon$.

\paragraph{Recovering isolated contacts.}
For $m=1,2,3$, \cref{alg:atomic-recovery} solves
\eqref{eq:atomic-contact-system} for the contact times $t_i$ and jerk
jumps $\Delta_i$, retaining admissible solutions as defined in
\cref{sec:canonical-constructions}.
To check clearance, minimize the polynomial $\norm{r(t)}^2-1$ on each
cubic piece using endpoint values and derivative roots, and reject any
negative minimum. The Green representation~\eqref{eq:reaction-green-reconstruction}
reconstructs the trajectory, and \cref{thm:contact-certificate}
certifies its optimality.

\paragraph{Constructing trials with an arc.}
For each of A, PA, AP, and PAP, the parameter vector $p$ contains the
ordered contact times, including arc entry and exit, the initial acceleration
and jerk, and the nonnegative contact masses $m_i$.
The bounds in \cref{prop:arc-recovery-bounds} give finite ranges for each
parameter; $\mathcal P_{\mathrm{rec}}$ collects the choices for all four
families, including coincident times and zero masses.
Starting from the prescribed initial state, we construct a trial $q_p$
by following cubic pieces, adding $m_iq_p(t_i)$ to the jerk at each
contact time, and integrating
$q_p^{(4)}(t)=\Lambda_p(t)q_p(t)$ along the arc, where
$\Lambda_p(t)=-3\norm{\ddot q_p(t)}^2
-4\inner{\dot q_p(t),q_p^{(3)}(t)}$.
When evaluating $\Lambda_pq_p$ for a trial, we limit each component of
position, velocity, acceleration, and jerk to a fixed range containing
every valid arc; exceeding this range counts as a violation.

\paragraph{Searching parameter intervals.}
A \emph{parameter box} assigns one closed interval with rational
endpoints to each coordinate of $p$ and represents all allowed parameter
choices within those intervals. \Cref{alg:arc-bounds} covers the search
ranges by such boxes.
For each trial, $\mathcal R(p)$ adds the squared endpoint and contact
errors to the largest violation of each remaining requirement:
clearance, the near-side condition, staying on the circle along the arc,
nonnegative multiplier, and the integration bounds.
Each term is nonnegative, so $\mathcal R(p)=0$ exactly when all
requirements hold.
We bound $\mathcal R$ over each box, including numerical integration
errors, and discard a box only if its lower bound is positive.
Thus a discarded box contains no valid solution.

The parameter $\eta$ bounds box diameters and numerical evaluation
errors. Among the retained boxes, we choose one trial $\widetilde r$
and bound its difference from every retained trial in position,
velocity, and acceleration by $\rho$; $I$ encloses all their costs.
These bounds contain the selected minimizer and its cost, even though
$\widetilde r$ itself need not be feasible.
\Cref{alg:canonical-recovery} halves $\eta$ until both error bounds
meet the requested accuracy.

\par\addvspace{6pt}
\begingroup
\ifdefined\nolinenumbers\nolinenumbers\fi
\noindent
\begin{minipage}[t]{0.50\linewidth}
  \vspace{0pt}
  \setlength{\intextsep}{0pt}
\begin{algorithm}[H]
  \small
  \raggedright
  \caption{}
  \label{alg:canonical-recovery}
  \renewcommand{\algorithmicrequire}{\textbf{Input:}}
  \renewcommand{\algorithmicensure}{\textbf{Output:}}
  \begin{algorithmic}[1]
    \Require $\omega$, $\varepsilon$
    \Ensure $\widetilde r,\rho,I$ with $\rho,\operatorname{width}(I)\leq\varepsilon$
    \If{$\omega$ is unblocked}
      \State \Return $\bar r,0,[J[\bar r],J[\bar r]]$
    \EndIf
    \For{$m=1,2,3$}
      \State $r\leftarrow{}$\cref{alg:atomic-recovery}$(\omega,m)$
      \If{$r\neq\varnothing$}
        \State \Return $r,0,[J[r],J[r]]$
      \EndIf
    \EndFor
    \For{$n=1,2,\ldots$}
      \State $(\widetilde r,\rho,I)\leftarrow{}$\cref{alg:arc-bounds}$(\omega,2^{-n})$
      \If{$\max\{\rho,\operatorname{width}(I)\}\leq\varepsilon$}
        \State \Return $\widetilde r,\rho,I$
      \EndIf
    \EndFor
  \end{algorithmic}
\end{algorithm}
\end{minipage}\hfill
\begin{minipage}[t]{0.45\linewidth}
  \vspace{0pt}
  \setlength{\intextsep}{0pt}
\begin{algorithm}[H]
  \small
  \raggedright
  \caption{}
  \label{alg:atomic-recovery}
  \renewcommand{\algorithmicrequire}{\textbf{Input:}}
  \renewcommand{\algorithmicensure}{\textbf{Output:}}
  \begin{algorithmic}[1]
    \Require $\omega$, $m$
    \Ensure Minimizer $r$, or $\varnothing$
    \State Solve~\eqref{eq:atomic-contact-system} for admissible $c$
    \State \Return $\varnothing$ if none exists
    \State Reconstruct $r$ by~\eqref{eq:reaction-green-reconstruction}
    \State \Return $r$
  \end{algorithmic}
\end{algorithm}
\par\vspace{10pt}
\begin{algorithm}[H]
  \small
  \raggedright
  \caption{}
  \label{alg:arc-bounds}
  \renewcommand{\algorithmicrequire}{\textbf{Input:}}
  \renewcommand{\algorithmicensure}{\textbf{Output:}}
  \begin{algorithmic}[1]
    \Require $\omega$, rational $\eta>0$
    \Ensure $\widetilde r,\rho,I$
    \State Cover $\mathcal P_{\mathrm{rec}}$ by boxes of diameter $\leq\eta$
    \ForAll{parameter boxes $C$}
      \State Bound $\mathcal R$ on $C$ to error $\eta$
      \State Discard $C$ if its lower bound is positive
    \EndFor
    \State Choose a retained trial $\widetilde r$
    \State Bound distances to $\widetilde r$ by $\rho$
    \State Enclose retained costs in $I$
    \State \Return $\widetilde r,\rho,I$
  \end{algorithmic}
\end{algorithm}
\end{minipage}
\par\endgroup\addvspace{6pt}   
\section{Discussion}
\label{sec:outlook}

We derive a complete analytical characterization of minimum-effort
collision-free motion for two robots with double-integrator dynamics and
zero initial and final velocities. The optimal relative trajectory combines cubic
pieces with isolated contacts or a boundary arc. We use the fact that every
optimal relative trajectory lies in a plane and the contact laws to reduce
the trajectory optimization problem to eight contact
families, whose admissible constructions partition the space of endpoint
configurations. We certify the global optimality of these constructions
through an exact comparison of trajectory costs.

These results provide structured candidate trajectories for use within
fleet-level planners. The companion work~\citep{kamm_contact_2026}
incorporates acceleration constraints and demonstrates collision-free planning
for one million robots in minutes and for thousands in a fraction of a second.

We identify three promising directions for future work. First, extending
the contact classification and global optimality certificate to pointwise
bounds on position, velocity, and acceleration would account for workspace
restrictions and robot operating limits. Second, studying different
dynamics would clarify which features of the optimal contact patterns
depend on the double-integrator model. Finally, we believe that a complete
characterization for nonzero initial and final velocities would be
particularly useful for online replanning, where robots must update their
trajectories while moving.

Altogether, this paper provides an analytical foundation for understanding
optimal two-robot motion and using its structure to design efficient
planners for larger robot fleets. 
\appendix
\crefalias{section}{appendix}
\section{Proofs}
\label{sec:proofs}
\leavevmode\par\nobreak

\noindent\textbf{Proof roadmap.}
The characterization follows from four steps:
\textbf{(i) Exhaustion.}
Every optimal relative trajectory lies in a plane. Together with side
selection, the contact bounds allow at most three
isolated contacts, or one boundary arc with at most one isolated contact
on each side, yielding the eight families of \cref{thm:canonical-catalogue}.
\textbf{(ii) Identification.}
The construction rules capture every selected minimizer. Conversely,
admissibility and the global certificate (\cref{thm:contact-certificate})
identify every admitted construction with the selected minimizer.
\textbf{(iii) Partition.}
Existence and exhaustion ensure coverage of $\mathcal D$.
Uniqueness of the selected trajectory and its multiplier prevents different
family labels at the same endpoint datum, establishing
\cref{thm:endpoint-regime-map}.
\textbf{(iv) Recovery.}
Exact algebraic tests handle isolated contacts. For the remaining families,
compactness and uniqueness force the certified trajectory and cost
enclosures to shrink, guaranteeing termination of
\cref{alg:canonical-recovery} (\cref{app:canonical-recovery}).

\subsection{Additional notation}
\label{app:additional-preliminaries}

We write $C^{2,1}([0,T];\mathbb{R}^d)$ for twice continuously
differentiable functions with a Lipschitz second derivative, and
$\mathrm{BV}((0,T);\mathbb{R}^d)$ for functions of bounded variation.
We use distributional derivatives and bounded-variation calculus as in
\citep[Sections~3.5 and~9.1]{folland1999}.
A finite nonnegative measure on $[0,T]$ has the
decomposition \citep[Section~3.5]{folland1999}
$\mu=\Lambda(t)\dd t
+\sum_k m_k\delta_{t_k}
+\mu_{\mathrm{sc}}$,
where $\Lambda\geq0$ is the density, $m_k=\mu(\{t_k\})$ are the
atomic masses, and $\mu_{\mathrm{sc}}$ is the singular-continuous part. 
\par
\ifdim\dimexpr\pagegoal-\pagetotal\relax<4\baselineskip
  \newpage
\fi
\subsection{Proofs for \cref{sec:two-robots}}
\label{app:proofs-section-3}

\subsubsection{Proof of \cref{prop:relative-multipliers}}

 \begin{lemma}[Interior stationarity]\label{lem:interior-stationarity}
Every local minimizer admits a locally finite nonnegative Radon measure
$\mu$ on $(0,T)$ satisfying~\eqref{eq:euler-lagrange}. \end{lemma}  

\begin{proof}[Proof of \cref{lem:interior-stationarity}]\sloppy
 Let $r$ be a local minimizer. For a nonnegative smooth test $\phi$
compactly supported in $(0,T)$, the radial variation
$(1+\varepsilon\phi)r$ is feasible for $\varepsilon\geq0$.
The distribution
$L[\phi]\coloneqq R^{-2}\int_0^T
\inner{\ddot r(t),\ddot{(\phi r)}(t)}\dd t$
is positive, hence is a locally finite nonnegative Radon measure
$\mu$. Indeed, a smooth cutoff $\chi\geq1$ on a fixed compact support
gives $\abs{L[\phi]}\leq L[\chi]\max\abs{\phi}$, so the Riesz
representation theorem applies locally. Both signs of the radial
variation are allowed away from contact, giving
$\operatorname{supp}(\mu)\subseteq\mathcal A\cap(0,T)$.
Every smooth compactly supported vector test $\eta$ decomposes as
$\eta=\phi r+\eta_\perp$, with
$\phi(t)=\inner{\eta(t),r(t)}/\norm{r(t)}^2$ and
$\inner{\eta_\perp(t),r(t)}=0$.
These components belong to $H^2((0,T);\mathbb R)$ and
$H^2((0,T);\mathbb R^d)$ and have compact support, since
$r\in H^2((0,T);\mathbb R^d)$ and $\norm{r(t)}\geq R>0$.
Since
$\norm{r(t)+\varepsilon\eta_\perp(t)}^2
=\norm{r(t)}^2+\varepsilon^2\norm{\eta_\perp(t)}^2$,
both signs are feasible, and local minimality gives
$\int_0^T\inner{\ddot r(t),\ddot\eta_\perp(t)}\dd t=0$.
Smooth approximation on a fixed compact subinterval extends the radial
identity to $\phi$: multiplication is continuous in
$H^2((0,T);\mathbb R)$, and convergence is uniform for integration
against the measure. Since $\norm{r(t)}=R$ on its support,
$\int_0^T\inner{\ddot r(t),\ddot\eta(t)}\dd t
=R^2\int_{0^+}^{T^-}\phi(t)\dd\mu(t)
=\int_{0^+}^{T^-}\inner{r(t),\eta(t)}\dd\mu(t)$. 
\end{proof}

\begin{proof}[Proof of \cref{prop:relative-multipliers}]\sloppy
 Strict endpoint separation makes $\mathcal A$ a compact subset of
$(0,T)$, so the measure of \cref{lem:interior-stationarity} is finite.
Extend it by zero to $[0,T]$. Density of smooth compactly supported
tests in $\mathcal H_T$ gives the stationarity identity, and
the support condition gives complementarity. 
\end{proof}

\subsubsection{Proof of \cref{prop:contact-structure}}
\label{app:proof-contact-structure}
\label{app:contact-identities}
\label{app:proof-contact-regularity}

We first establish the regularity and clearance identity needed to
exclude accumulating contacts, using the stationary pair $(r,\mu)$ of
\cref{sec:necessary}.

 \begin{proposition}[Trajectory regularity]\label{prop:contact-regularity}
Let $r$ be a feasible trajectory and $\mu$ a finite nonnegative Radon
measure satisfying~\eqref{eq:euler-lagrange}. The distributional
derivative $D_t^3r$ is represented by a right-continuous function
$r^{(3)}\in\mathrm{BV}((0,T);\mathbb{R}^d)$, which we call the jerk.
The weak acceleration $D_t^2r$ has a Lipschitz continuous
representative $\ddot{r}$ on $[0,T]$, with
$\norm{\ddot{r}(t)-\ddot{r}(s)}
\leq \norm{r^{(3)}}_{L^\infty((0,T);\mathbb{R}^d)}\,\abs{t-s}$
for $s,t\in[0,T]$.
Consequently, $r\in C^{2,1}([0,T];\mathbb{R}^d)$ and
$g\in C^{2,1}([0,T];\mathbb{R})$. \end{proposition}  

\begin{proof}[Proof of \cref{prop:contact-regularity}]\sloppy
 By the Sobolev embedding
$H^2((0,T);\mathbb{R}^d)\hookrightarrow C^1([0,T];\mathbb{R}^d)$
\citep[Theorem~4.12, Part~II]{adams2003}, $r$ is bounded and continuous.
The primitive
$F_\mu(t)\coloneqq\int_{0^+}^{t}r(s)\dd\mu(s)$
is right-continuous with bounded variation and satisfies
$D_t^1F_\mu=r\,\mu$ \citep[Sections~3.5 and~9.1]{folland1999}.
Thus~\eqref{eq:euler-lagrange} gives
$r^{(3)}(t)=r^{(3)}_0+F_\mu(t)$ for a constant vector
$r^{(3)}_0\in\mathbb R^d$.
In particular,
$r^{(3)}\in\mathrm{BV}((0,T);\mathbb{R}^d)\cap L^\infty((0,T);\mathbb{R}^d)$.
Integrating once more gives
$\ddot r(t)=\ddot r_0+\int_0^t r^{(3)}(s)\dd s$,
which proves the stated Lipschitz bound and
$r\in C^{2,1}([0,T];\mathbb{R}^d)$.
Finally,
$\ddot g(t)=-2\norm{\dot r(t)}^2-2\inner{r(t),\ddot r(t)}$
is Lipschitz, so $g\in C^{2,1}([0,T];\mathbb{R})$. 
\end{proof}

At every interior contact, $g$ attains a local maximum of zero, which gives
tangency: ${\dot{g}(t)=-2\inner{ r(t),\dot{r}(t)}=0}$ for
$t\in\mathcal{A}\cap(0,T)$.

Let $t^*\in(0,T)$ be an isolated contact or an endpoint of a boundary
arc. The acceleration is continuous by \cref{prop:contact-regularity},
so $\jmp{\ddot r}_{t^*}=0$. The atom of
$D_t^1 r^{(3)}=r\,\mu$ gives the jerk jump
$\jmp{r^{(3)}}_{t^*}=\mu(\{t^*\})r(t^*)$
\citep[Section~3.5]{folland1999}.

On every open interval disjoint from $\operatorname{supp}(\mu)$,
$D_t^4r=0$, so the trajectory is cubic.
A free segment between two interior contacts is a \emph{free bridge}.
For the clearance $g$ of \cref{sec:necessary}, define the
\emph{clearance forcing}
$f_g(t)\coloneqq6\norm{\ddot r(t)}^2+8\inner{\dot r(t),r^{(3)}(t)}$.

 \begin{lemma}[Clearance identity]\label{lem:clearance-identity}
The forcing $f_g$ has a Lipschitz continuous representative on $[0,T]$, and
\refstepcounter{equation}\label{eq:clearance-identity}%
$D_t^4g=-f_g(t)\dd t-2R^2\dd\mu$~\textup{(\theequation)}
in distributions on $(0,T)$.
On every boundary arc $[a,b]$,
$\mu|_{(a,b)}=-f_g(t)/(2R^2)\dd t$, with $f_g\leq0$ on $[a,b]$. \end{lemma}  

\begin{proof}[Proof of \cref{lem:clearance-identity}]\sloppy
 By \cref{prop:contact-regularity} and the Lebesgue--Stieltjes product
rule \citep[Theorem~3.36]{folland1999}, $D_t^3g$ has the
right-continuous representative
$g^{(3)}(t)=-6\inner{\dot r(t),\ddot r(t)}
-2\inner{r(t),r^{(3)}(t)}$
in $\mathrm{BV}((0,T);\mathbb{R})$.
Differentiating once more and using $D_t^1r^{(3)}=r\,\mu$ gives
$-D_t^1g^{(3)}
=(6\norm{\ddot r(t)}^2+8\inner{\dot r(t),r^{(3)}(t)})\dd t
+2\norm{r(t)}^2\dd\mu(t)$.
Since $\norm{r(t)}=R$ on $\operatorname{supp}(\mu)$, this is
\eqref{eq:clearance-identity}.
The same product rule gives
$D_t^1f_g=20\inner{\ddot r(t),r^{(3)}(t)}\dd t
+8\inner{\dot r(t),r(t)}\dd\mu(t)$.
Tangency on $\operatorname{supp}(\mu)$ removes the measure term.
The remaining density is bounded, so $f_g$ has a Lipschitz continuous
representative.
On a boundary arc $(a,b)$, $g=0$, so
\eqref{eq:clearance-identity} gives
$\mu=-f_g(t)/(2R^2)\dd t$ there.
Nonnegativity of $\mu$ and continuity of $f_g$ yield $f_g\leq0$ on $[a,b]$. 
\end{proof}

On a boundary arc $[a,b]$, the continuous density gives
$r^{(4)}(t)=\Lambda(t)r(t)$ classically for $t\in(a,b)$.
Its arc-side endpoint traces are $\Lambda(a^+)$ and $\Lambda(b^-)$.

\begin{proof}[Proof of \cref{prop:contact-structure}]\sloppy
 Use \cref{prop:contact-regularity,lem:clearance-identity}.

\emph{Exclusion of contact accumulation.}
Suppose a sequence of contacts $(t_n)_{n\in\mathbb{N}}$, with
$t_n\uparrow t_{\mathrm{acc}}$, accumulates without
forming a boundary interval to the left of $t_{\mathrm{acc}}$.
Tangency gives $g(t_n)=\dot g(t_n)=0$.
Rolle's theorem \citep[Theorem~2.3.8]{trench_real_analysis_2003}
gives a sequence $(s_n)_{n\in\mathbb{N}}$ with
$s_n\in(t_n,t_{n+1})$ and
$\ddot{g}(s_n)=0$. As $s_n\uparrow t_{\mathrm{acc}}$, continuity gives
$g(t_{\mathrm{acc}})=\dot{g}(t_{\mathrm{acc}})=\ddot{g}(t_{\mathrm{acc}})=0$.
Finally, $g^{(3)}(t_{\mathrm{acc}}^-)=0$; otherwise $g^{(3)}$ would have a fixed
sign near $t_{\mathrm{acc}}$, making $\ddot{g}$ strictly monotone there and
precluding the zeros $s_n\to t_{\mathrm{acc}}$.

Let $r^{(3)}_*\coloneqq r^{(3)}(t_{\mathrm{acc}}^-)$ and let $r_{\mathrm{cub}}$ be the free
cubic with position $r(t_{\mathrm{acc}})$, velocity $\dot{r}(t_{\mathrm{acc}})$,
acceleration $\ddot{r}(t_{\mathrm{acc}})$, and jerk $r^{(3)}_*$ at $t_{\mathrm{acc}}$.
Write $\dot f_*\coloneqq20\inner{\ddot{r}(t_{\mathrm{acc}}),r^{(3)}_*}$ for the left
limit of the density of $D_t^1 f_g$.
The clearance
$g_{\mathrm{cub}}(s)\coloneqq R^2-\norm{r_{\mathrm{cub}}(t_{\mathrm{acc}}+s)}^2$
has vanishing derivatives through order three at zero, so
\eqref{eq:clearance-identity} gives
$g_{\mathrm{cub}}(s)
=-\frac{f_g(t_{\mathrm{acc}})}{24}s^4-\frac{\dot f_*}{120}s^5
-\frac{\norm{r^{(3)}_*}^2}{36}s^6$.
Thus $g_{\mathrm{cub}}$ is positive or nonpositive throughout some
interval $(-\varepsilon,0)$.

Suppose first that $g_{\mathrm{cub}}>0$ on $(-\varepsilon,0)$. The expansion
then gives $f_g(t_{\mathrm{acc}})<0$, or $f_g(t_{\mathrm{acc}})=0$ and
$\dot f_*>0$. In the second case the identity for $D_t^1 f_g$ gives
$f_g(t)=-\int_t^{t_{\mathrm{acc}}}20\inner{\ddot{r}(u),r^{(3)}(u)}\dd u<0$
for $t<t_{\mathrm{acc}}$ close to $t_{\mathrm{acc}}$. Either way, $f_g<0$ on a
left neighborhood of $t_{\mathrm{acc}}$. Fix a contact $t_{n_0}$ there.
Since $g=0$ on $\operatorname{supp}(\mu)$ and $g=\dot g=0$ at
the contact endpoints, two integrations by parts in
\eqref{eq:clearance-identity} give, for every later contact $t_n$,
$\int_{t_{n_0}}^{t_n}(-g(t))f_g(t)\dd t
=\int_{t_{n_0}}^{t_n}(\ddot g(t))^2\dd t$.
The left side is nonpositive and the right side nonnegative, so both
vanish. Hence $g=0$ on $[t_{n_0},t_n]$; letting
$t_n\uparrow t_{\mathrm{acc}}$ gives the excluded boundary interval.

Otherwise $g_{\mathrm{cub}}\leq0$ on $(-\varepsilon,0)$: the limiting free
cubic is feasible near $t_{\mathrm{acc}}$.
Integrating~\eqref{eq:euler-lagrange} backwards from $t_{\mathrm{acc}}$ gives
$r(t)
=
r_{\mathrm{cub}}(t)
+
\int_{t^+}^{t_{\mathrm{acc}}^-}
\frac{(s-t)^3}{6}\,r(s)\dd\mu(s)$,
$t<t_{\mathrm{acc}}$.
For $t$ sufficiently close to $t_{\mathrm{acc}}$, the kernel is positive and
$\inner{ r_{\mathrm{cub}}(t),r(s)}>0$. Hence any mass of $\mu$ in
$(t,t_{\mathrm{acc}})$ pushes $r(t)$ strictly farther from the origin than
$r_{\mathrm{cub}}(t)$: if $\mu((t,t_{\mathrm{acc}}))>0$, then
$\norm{r(t)}>\norm{r_{\mathrm{cub}}(t)}\geq R$.
A contact sufficiently close to $t_{\mathrm{acc}}$ therefore cannot have
mass of $\mu$ between itself and $t_{\mathrm{acc}}$. Along the accumulating
contacts this yields $t_0<t_{\mathrm{acc}}$ such that
$\mu((t_0,t_{\mathrm{acc}}))=0$. Hence $r=r_{\mathrm{cub}}$ and
$g(t)=g_{\mathrm{cub}}(t-t_{\mathrm{acc}})$ on $(t_0,t_{\mathrm{acc}})$.
The accumulating contacts are zeros of this polynomial, so $g=0$ on
$(t_0,t_{\mathrm{acc}})$, again a boundary interval contrary to the assumption.
The same argument applies from the right and at the endpoints of the time horizon.
Thus contacts can accumulate only into boundary arcs.

\emph{Finitely many contact components.}
If $\mathcal A$ had infinitely many components, compactness would give
a monotone sequence of points from distinct components converging to
$t_{\mathrm{acc}}$. The preceding argument places its tail in one boundary
interval, a contradiction. Thus $\mathcal A$ has finitely many components,
each an isolated point or a closed interval by closedness.

\emph{Multiplier decomposition.}
By \cref{lem:clearance-identity}, $\mu$ has nonnegative Lipschitz
density $-f_g/(2R^2)$ on each arc interior. Complementarity confines
the remaining support to finitely many isolated contacts and arc
endpoints, so the remaining part is atomic and $\mu_{\mathrm{sc}}=0$. 
\end{proof}

\emph{Remark on related arguments.}
For a scalar minimum-acceleration problem,
\citet[proof of Theorem~2.1]{dontchev_kolmanovsky_1995}
show that two contacts with the same linear bound force contact
throughout the intervening interval.
A linear bound has $f_g=0$, whereas disk clearance allows either sign;
their comparison therefore leaves contact accumulation open here.
\citet[Proposition~3.6 and Equation~(7.17)]{bonnans_hermant_2009}
also decompose multipliers into arc densities and atoms at contact times.
They assume finitely many entries, exits, and isolated contacts
in~(A4)(i); here we prove finiteness.  
\subsubsection{Proof of \cref{thm:contact-certificate}}
\label{app:proof-contact-certificate}
\label{app:spline-identities}
\label{app:proof-acceleration-orthogonality}

Recall the free cubic $\bar r$ and correction space $\mathcal H_T$
of \cref{sec:certificate}.

 \begin{lemma}[Green kernel identities]\label{lem:green-kernel}
The Green kernel~\eqref{eq:relaxation-kernel} is strictly positive on
$(0,T)^2$. For $s\in(0,T)$, it satisfies
$D_t^4G_T(\cdot,s)=\delta_s$ and has zero value and first derivative at
both endpoints. Its evaluation identity, acting componentwise, is
$h(s)=\int_0^T\frac{\partial^2G_T}{\partial t^2}(t,s)\,\ddot{h}(t)\dd t$
for $h\in\mathcal H_T$. \end{lemma}  

\begin{proof}[Proof of \cref{lem:green-kernel}]\sloppy
 Positivity follows from $3Ts-Tt-2ts\geq2s(T-s)>0$ for $0<t\leq s<T$.
Direct differentiation of~\eqref{eq:relaxation-kernel} on either side
of $t=s$ gives cubic pieces, matching values through the second
derivative, and a unit jump in the third derivative. Integrating by
parts separately on $(0,s)$ and $(s,T)$ eliminates the clamped endpoint
terms, and the unit jump gives the evaluation identity. 
\end{proof}

\begin{proof}[Proof of \cref{thm:contact-certificate}]\sloppy
 Twice integrating by parts, using the clamped endpoint conditions and
$\bar r^{(4)}=0$, gives
$\int_0^T\inner{\ddot{\bar r}(t),\ddot h(t)}\dd t=0$ for
$h\in\mathcal H_T$, hence the cost decomposition in \cref{sec:certificate}.
The evaluation identity in \cref{lem:green-kernel} and
finiteness of $\mu$ give the Green
reconstruction~\eqref{eq:reaction-green-reconstruction}.
The integral belongs to $\mathcal H_T$. Indeed, each scalar kernel
section has norm $\sqrt{G_T(s,s)}$, uniformly bounded for
$s\in[0,T]$, and $r$ is bounded. Four distributional derivatives give
$r\,\mu$; subtracting it from $r-\bar r$ leaves a clamped cubic,
which is zero. Equivalently, by the reproducing identity and
orthogonality of the free cubic,
$\int_0^T\inner{\ddot{r}(t),\ddot{h}(t)}\dd t
=\int_{0^+}^{T^-}\inner{ r(t),h(t)}\dd\mu(t)$,
$h\in\mathcal H_T$.
Expanding $\frac12\norm{\ddot r+\ddot h}^2$, using
$2\inner{ r(t),h(t)}=\norm{r(t)+h(t)}^2-\norm{r(t)}^2-\norm{h(t)}^2$,
and using $\norm{r(t)}=R$ on the support proves the cost identity~\eqref{eq:cost-identity}. Its last integral is nonnegative
for a feasible $q$.

For the projection test, the bounded trace map
$B:\mathcal H_T\to L^2((0,T),\mu;\mathbb R^d)$,
$Bh=h\big|_{\operatorname{supp}(\mu)}$,
has adjoint
$B^*f=\int_{0^+}^{T^-}G_T(\cdot,s)f(s)\dd\mu(s)$.
Thus $S=BB^*$ acts componentwise by the scalar kernel $G_T$.
On $\operatorname{supp}(\mu)$, projection of
\eqref{eq:reaction-green-reconstruction} gives
$Sz=z-\bar z$ and $0<z\leq R$.
If $\mu\neq0$, strict positivity of $G_T$ gives $Sz>0$ there,
so $\delta<R$ and $Sz\leq(1-\delta/R)z$.
For a scalar $f\in L^2((0,T),\mu;\mathbb R)$, the inequality
$2\abs{f(t)f(s)}\leq
\abs{f(t)}^2\frac{z(s)}{z(t)}
+\abs{f(s)}^2\frac{z(t)}{z(s)}$
and symmetry of the positive kernel imply
$\abs{\inner{ f,Sf}_{L^2((0,T),\mu;\mathbb R)}}
\leq\int_{0^+}^{T^-}\abs{f(t)}^2\frac{(Sz)(t)}{z(t)}\dd\mu(t)
\leq(1-\kappa_{\mathrm{cert}})\norm f_{L^2((0,T),\mu;\mathbb R)}^2$.
Tonelli's theorem justifies this weighted estimate even if $z$
has no positive uniform lower bound on the support. Applying it
componentwise gives $\norm S\leq1-\kappa_{\mathrm{cert}}$; the Hilbert-space
identity $\norm B^2=\norm{BB^*}$ therefore yields
$\int_{0^+}^{T^-}\norm{h(t)}^2\dd\mu(t)
\leq(1-\kappa_{\mathrm{cert}})\norm h_{\mathcal H_T}^2$.
This proves the bounds in \cref{thm:contact-certificate}. When $\mu=0$,
\eqref{eq:reaction-green-reconstruction} gives $r=\bar r$ and
$Q_\mu[h]=\norm h_{\mathcal H_T}^2$ directly. Finally, zero
acceleration norm for a clamped difference forces that difference
to vanish, proving uniqueness whenever $\kappa_{\mathrm{cert}}>0$ and also in
the free case. 
\end{proof} 

\subsubsection{Interior multiplier and terminal contact}
\label{app:proof-canonical-reaction}
\label{app:contact-invariants}
\label{app:proof-contact-invariants}

From here on, use the normalized endpoint data of
\cref{sec:zero-endpoint-velocities}.
Before proving that every optimal relative trajectory lies in a plane,
we show that the interior multiplier remains
finite at a terminal contact with zero velocity. The argument uses the following
conserved quantities.

 \begin{proposition}[Noether invariants]\label{prop:contact-invariants}
Let $r$ be feasible and let $\mu$ be a finite nonnegative Radon
measure satisfying~\eqref{eq:euler-lagrange}. With $r^{(3)}$ denoting the
right-continuous jerk, the quantities
$\boldsymbol\Gamma\coloneqq \dot{r}(t)\ddot{r}(t)^{\mathsf T}
    -\ddot{r}(t)\dot{r}(t)^{\mathsf T}
    -r(t)r^{(3)}(t)^{\mathsf T}+r^{(3)}(t)r(t)^{\mathsf T}$ and
$E\coloneqq\tfrac12\norm{\ddot{r}(t)}^2
    -\inner{\dot{r}(t),r^{(3)}(t)}$
are constant on $(0,T)$. In particular, their values agree across
isolated contacts and boundary-arc endpoints. \end{proposition}  

\begin{proof}[Proof of \cref{prop:contact-invariants}]\sloppy
 By \cref{prop:contact-regularity}, the bounded-variation product rule
\citep[Theorem~3.36]{folland1999} applies entrywise and gives
$D_t^1\boldsymbol\Gamma
=(D_t^4r)r^{\mathsf T}-r(D_t^4r)^{\mathsf T}
=(rr^{\mathsf T}-rr^{\mathsf T})\,\mu=0$
and $D_t^1E=-\inner{\dot r,D_t^4r}
=-\inner{\dot r(t),r(t)}\dd\mu(t)=0$.
The last equality follows from tangency on $\mathcal A\cap(0,T)$
and the support condition for $\mu$.
Both right-continuous representatives are therefore constant on $(0,T)$. 
\end{proof}

 \begin{proposition}[Interior multiplier]\label{prop:canonical-reaction}
Every global minimizer has initial and terminal free pieces of positive
duration and a unique finite nonnegative interior multiplier measure
satisfying~\eqref{eq:euler-lagrange}, compactly supported in $(0,1)$,
including at a terminal contact. \end{proposition}  

\begin{proof}[Proof of \cref{prop:canonical-reaction}]\sloppy
 Let $r$ be a global minimizer. \Cref{lem:interior-stationarity} gives
its locally finite interior multiplier $\mu$.

\paragraph{Terminal contact with zero velocity.}
On compact interior intervals, the proofs of
\cref{prop:contact-regularity,prop:contact-invariants} apply to
the locally finite measure. Thus $\ddot{r}$ is locally Lipschitz, the
right-continuous jerk $r^{(3)}$ has locally bounded variation, and
$E$ and $\boldsymbol\Gamma$ are constants.
At any interior contact write
$r(t)=n$ and $\ddot{r}(t)=An+B$ with $\inner{B,n}=0$; tangency gives
$\inner{\dot{r}(t),n}=0$. The component of $r^{(3)}(t)$ orthogonal to $n$
is $\boldsymbol\Gamma n-A\dot{r}(t)$, and eliminating it from $E$ gives
$(A+\norm{\dot{r}(t)}^2)^2+\norm B^2
=\norm{\dot{r}(t)}^4+2E+2\inner{\dot{r}(t),\boldsymbol\Gamma n}$.
The velocity is bounded on $[0,1]$, so this identity bounds $\ddot{r}$
at all interior contacts. On a free component $\ddot{r}$ is affine;
its contact-end values give the same bound throughout. A component
meeting an endpoint is a single polynomial there. Initial strict
separation supplies an initial free interval. These observations
bound $\ddot{r}$ on the entire open interval without assuming finitely
many contact components.

If $\norm{r_T}>1$, strict separation already gives a final free
interval. Suppose instead $Y=r_T$ has unit norm. Choose $s<1$
so that $y(t)\coloneqq\inner{ Y,r(t)}\geq1/2$ for $s<t<1$.
Use the right-continuous representative $y^{(3)}(t)=\inner{ Y,r^{(3)}(t)}$.
Then $D_t^1 y^{(3)}=y\,\mu\geq\tfrac12\mu$.
The bounded function $\ddot{y}$ is convex and $y^{(3)}$ is nondecreasing.
Hence $y^{(3)}\in L^1((s,1);\mathbb R)$: its negative part is bounded
and its integral is a difference of bounded values of $\ddot{y}$.
For $s<t<1$, increasing $s$ if necessary to fix its representative,
$\norm{r^{(3)}(t)-r^{(3)}(s)}
\leq\int_{s^+}^{t}\norm{r(u)}\dd\mu(u)
=\mu((s,t])\leq2\bigl(y^{(3)}(t)-y^{(3)}(s)\bigr)$.
Thus $r^{(3)}\in L^1((s,1);\mathbb R^d)$ and $\ddot{r}$ has a continuous,
absolutely continuous extension to $1$. Monotonicity and integrability
also give $(1-t)\abs{y^{(3)}(t)}\to0$, and the preceding bound gives
$(1-t)\norm{r^{(3)}(t)}\to0$. Since $\dot{r}(1)=0$ and $\ddot{r}$ is bounded,
$\norm{\dot{r}(t)}=O(1-t)$, whence $\inner{ \dot{r}(t),r^{(3)}(t)}\to0$.
Integrating the energy invariant and using the zero initial and final
velocities now yields
$E=\tfrac32\int_0^1\norm{\ddot{r}(t)}^2\dd t>0$,
$\norm{\ddot{r}(1)}^2=2E$.
The strict inequality holds here because $r_0\neq Y$.

Use the positive gap $b(t)=\norm{r(t)}^2-1=-g(t)$. Locally, the same
product rule as in~\eqref{eq:clearance-identity} gives
$D_t^4b=f_g\dd t+2\mu$,
$f_g(t)=10\norm{\ddot{r}(t)}^2-8E\to12E>0$ as $t\uparrow1$.
If interior contacts accumulated at $1$, their tangent zeros of
$b$ would, by Rolle's theorem, give infinitely many zeros of the
continuous function $\ddot{b}$. But the preceding measure
identity makes this function strictly convex near $1$, where it
can have at most two zeros. A final boundary interval is excluded
by the same identity. There is therefore a final free interval
of positive duration. The multiplier is supported in a compact
subset of $(0,1)$ and is finite. Its uniqueness follows from
$D_t^4r=r\,\mu$ and $\norm{r(t)}=1$ on its support. Possible endpoint
atoms do not enter the interior equation or the clamped Green
reconstruction, and are omitted. 
\end{proof}

\subsubsection{Proof of \cref{prop:planar-reduction,prop:side-selection}}
\label{app:proof-planar-reduction}
\label{app:proof-side-selection}

For off-axis data, recall the normal
\refstepcounter{equation}\label{eq:canonical-projection}%
$e\coloneqq(r_{T,2},d_0-r_{T,1})/\norm{r_T-r_0}$ and offset
$\delta\coloneqq\inner{ e,r_0}=\inner{ e,r_T}
=d_0r_{T,2}/\norm{r_T-r_0}$~\textup{(\theequation)}.

\begin{proof}[Proof of \cref{prop:planar-reduction,prop:side-selection}]\sloppy
 If the free cubic is feasible, \cref{thm:contact-certificate} at $\mu=0$
makes it the unique global minimizer. Its image lies in the span of the
endpoints and, for off-axis data, in the near side.
Hence assume blocked data, so that $r_0\neq r_T$, and put
$a\coloneqq(r_0-r_T)/\norm{r_0-r_T}$. For off-axis data use the normal
$e$ and offset $\delta$ of~\eqref{eq:canonical-projection}; for symmetric
data let $e\perp a$ be any unit vector and $\delta=0$. In both cases
$r_0-\inner{a,r_0}a=r_T-\inner{a,r_T}a=\delta e$.

\paragraph{Transverse comparison.}
Let $r\in\mathcal M^*$ and write $r=Xa+Y$, with $X=\inner{a,r}$ and
$\inner{a,Y(t)}=0$. The endpoint states give
$Y-\delta e\in\mathcal H_1$ and
$\norm{\ddot r}^2=\norm{\ddot X}^2+\norm{Y-\delta e}_{\mathcal H_1}^2$.
Let
$\mathcal V=\{w\in H^2((0,1);\mathbb R):
w(0)=w(1)=\dot{w}(0)=\dot{w}(1)=0\}$,
$\norm w_{\mathcal V}^2=\int_0^1\abs{\ddot{w}(t)}^2\dd t$.
Put $\mathcal I=\{t\in(0,1):\abs{X(t)}<1\}$, and write
$\xi(t)=\sqrt{1-X(t)^2}$ on its closure. Feasibility requires
$\norm{Y}\geq\xi$ on $\overline{\mathcal I}$.
For $n\geq3$, let
$\mathcal I_n=\overline{\mathcal I\cap[1/n,1-1/n]}$ and minimize
$\norm w_{\mathcal V}^2$ subject to $w\geq\xi-\delta$ on $\mathcal I_n$.
These compact constraints admit a strictly feasible smooth cutoff.
Their unique minimizer $u_n$ therefore has a nonnegative finite multiplier
$\nu_n$ supported on $\mathcal I_n$, with
$u_n=\int_{0^+}^{1^-}G_1(\cdot,s)\dd\nu_n(s)$,
$\norm{u_n}_{\mathcal V}^2
=\int_{0^+}^{1^-}(\xi(s)-\delta)\dd\nu_n(s)$.
The strict Slater condition gives the multiplier
\citep[Theorem~1.1]{colonius1983}. The second identity follows from
complementarity and the reproducing property. If $\mathcal I_n$ is empty, take
$u_n=0$ and $\nu_n=0$. Positivity of $G_1$ gives $u_n\geq0$.

Let $\hat Y(s)\coloneqq Y(s)/\norm{Y(s)}$ where $Y(s)\neq0$, and
$\hat Y(s)\coloneqq e$ otherwise, and set
$W_n=\int_{0^+}^{1^-}G_1(\cdot,s)\hat Y(s)\dd\nu_n(s)$.
Since $\inner{Y-\delta e,\hat Y}\geq\xi-\delta\inner{e,\hat Y}$ on
$\mathcal I_n$, the reproducing property and complementarity give
$\inner{Y-\delta e,W_n}_{\mathcal H_1}
\geq\norm{u_n}_{\mathcal V}^2
+\delta\int_{0^+}^{1^-}(1-\inner{e,\hat Y(s)})\dd\nu_n(s)$.
Expanding $\norm{Y-\delta e-W_n}_{\mathcal H_1}^2$ yields the transverse
comparison
$\norm{Y-\delta e}_{\mathcal H_1}^2-\norm{u_n}_{\mathcal V}^2
\geq\norm{Y-\delta e-W_n}_{\mathcal H_1}^2
+\norm{u_n}_{\mathcal V}^2-\norm{W_n}_{\mathcal H_1}^2
+2\delta\int_{0^+}^{1^-}(1-\inner{e,\hat Y(s)})\dd\nu_n(s)$.
The middle difference equals
$\frac12\int_{0^+}^{1^-}\int_{0^+}^{1^-}
G_1(s,t)\norm{\hat Y(s)-\hat Y(t)}^2\dd\nu_n(s)\dd\nu_n(t)\geq0$.
Thus $\norm{u_n}_{\mathcal V}\leq\norm{Y-\delta e}_{\mathcal H_1}$.
The constraint sets are nested, and the minimum-norm variational
inequality gives, for $m\geq n$,
$\norm{u_m-u_n}_{\mathcal V}^2\leq\norm{u_m}_{\mathcal V}^2-\norm{u_n}_{\mathcal V}^2$.
Hence $u_n$ converges strongly in $\mathcal V$ to a nonnegative $u_+$
satisfying every upper constraint. The trajectory $Xa+(\delta+u_+)e$ is
feasible and has the endpoint states of $r$. Minimality of $r$ gives
$\norm{Y-\delta e}_{\mathcal H_1}=\norm{u_+}_{\mathcal V}$, so every
nonnegative term on the right of the transverse comparison tends to zero.

If $\delta>0$, then
$\int_{0^+}^{1^-}\norm{\hat Y(s)-e}^2\dd\nu_n(s)
=2\int_{0^+}^{1^-}(1-\inner{e,\hat Y(s)})\dd\nu_n(s)\to0$.
With $M=\sup_{s\in(0,1)}\sqrt{G_1(s,s)}<\infty$, the reproducing property
gives $u_n\leq M\norm{u_n}_{\mathcal V}$, and positivity and symmetry of
the kernel give
$\norm{W_n-u_ne}_{\mathcal H_1}^2
\leq\int_{0^+}^{1^-}\norm{\hat Y(s)-e}^2u_n(s)\dd\nu_n(s)\to0$.
Together with $\norm{Y-\delta e-W_n}_{\mathcal H_1}\to0$, this proves
$Y=(\delta+u_+)e$. Thus $r$ lies in
$\operatorname{span}\{a,e\}=\operatorname{span}\{r_0,r_T\}$, and
$\inner{e,r(t)}\geq\delta>0$ throughout.
With the multiplier of \cref{prop:canonical-reaction},
\cref{thm:contact-certificate} therefore gives uniqueness.

For blocked symmetric data, $\overline{\mathcal I}$ is compactly
contained in $(0,1)$. This follows immediately at strict endpoints.
At a terminal contact $r_T=-a$, the final free cubic of
\cref{prop:canonical-reaction} gives
$X(t)=-1+\tfrac12\ddot{X}(1)(1-t)^2-\tfrac16X^{(3)}(1^-)(1-t)^3$ and
$Y(t)=O((1-t)^2)$.
Separation forces $\ddot{X}(1)\leq0$ and, if it vanishes, $X^{(3)}(1^-)\geq0$.
Hence $X\leq-1$ near the terminal endpoint. For all sufficiently large
$n$, the set $\mathcal I_n$ is therefore the fixed compact set
$\overline{\mathcal I}$, and $u_n=u_+$ with a fixed multiplier $\nu$.
The middle difference in the transverse comparison vanishes, and $G_1$
is bounded below by a positive constant on
$\overline{\mathcal I}\times\overline{\mathcal I}$. Hence $\hat Y=b_0$
$\nu$-almost everywhere for a unit vector $b_0\perp a$, so $W_n=u_+b_0$
and $Y=u_+b_0$. The multiplier is nonzero because $X$ crosses zero, where
$u_+\geq1$, and strict kernel positivity gives $u_+(t)>0$ for every
$t\in(0,1)$. Thus $r$ lies in $\operatorname{span}\{a,b_0\}$, which
contains $r_0$ and $r_T$.

\paragraph{Equality for symmetric data.}
Write this minimizer as $r=Xa+zb_0$, with $z=u_+>0$ in $(0,1)$.
Its multiplier $\mu$ from \cref{prop:canonical-reaction} is nonzero,
\cref{thm:contact-certificate} applies with unit vector $b_0$ and
$\delta=0$,
and projection of
\eqref{eq:reaction-green-reconstruction} gives $Bz=S[Bz]$.
Here $B:\mathcal V\to L^2((0,1),\mu;\mathbb R)$ is the scalar trace map
from the certificate and $S=BB^*$ has kernel $G_1$.
Compactness of $\operatorname{supp}(\mu)$ makes $z$ bounded away from zero there.
For $f\in L^2((0,1),\mu;\mathbb R)$, symmetry gives
$\inner{ f,(I-S)f}_{L^2((0,1),\mu;\mathbb R)}
=\frac12\int_{0^+}^{1^-}\int_{0^+}^{1^-}
G_1(t,s)z(t)z(s)
\left(\frac{f(t)}{z(t)}-\frac{f(s)}{z(s)}\right)^2
\dd\mu(s)\dd\mu(t)$.
Strict kernel positivity makes the nullspace of $I-S$ the span of
the trace of $z$. If a scalar $h\in\mathcal V$ has $Q_\mu[h]=0$,
positivity of $I-B^*B$ gives $h=B^*Bh$, and therefore $Bh=S[Bh]$.
Since $B^*[Bz]=z$, every such $h$ is a scalar multiple of $z$.

Any other global minimizer $q$ has equality in the cost identity~\eqref{eq:cost-identity}. Componentwise application of this
scalar equality case yields $q=r+z\zeta$ for a constant $\zeta\in\mathbb R^d$.
The remaining term in that identity forces $\norm{q(t)}=1$ on $\operatorname{supp}(\mu)$. With $\zeta_1=\inner{ a,\zeta}$ and $\zeta_2=\inner{ b_0,\zeta}$,
$2\zeta_1\frac{X(t)}{z(t)}+2\zeta_2+\norm{\zeta}^2=0$ on $\operatorname{supp}(\mu)$.
If the support contains two distinct contact normals, their ratios $X/z$
differ. Thus $\zeta_1=0$ and $\norm{b_0+\zeta}=1$, so $q=Xa+zb$ for the fixed
unit vector $b=b_0+\zeta\perp a$.

If the contact normal on $\operatorname{supp}(\mu)$ is one fixed $n=n_1a+n_2b_0$ with $n_2>0$,
Green reconstruction gives $r=\bar r+nA$, where
$A(t)=\int_{0^+}^{1^-}G_1(t,s)\dd\mu(s)$.
At a support time the transverse coordinate forces $A=1$, hence
$\bar r=0$. The free longitudinal coordinate is strictly decreasing
in $(0,1)$, so it has a single zero $t_*$. The multiplier is one atom and
$r(t)=\bar r(t)+\frac{G_1(t,t_*)}{G_1(t_*,t_*)}n$.
Since $z=n_2A$, every feasible trajectory with the same cost has the same representation
with a unit vector $\widetilde n=n+n_2\zeta$ in place of $n$.
Feasibility gives tangency at $t_*$:
$\inner{\dot{\bar r}(t_*),\widetilde n}
+\frac{\partial_1G_1(t_*,t_*)}{G_1(t_*,t_*)}=0$.
The nonzero free longitudinal derivative fixes
$\inner{ a,\widetilde n}=n_1$. Its unit norm then gives
$\widetilde n=n_1a+n_2b$ with $b\perp a$ unit, again yielding $q=Xa+zb$.

Conversely, every such choice of $b$ preserves the endpoint states,
clearance, and cost. Thus the complete family of global minimizers is
$r_b(t)=X(t)a+z(t)b$ for $b\in a^\perp$ with $\norm b=1$.
Each lies in the fixed plane $\operatorname{span}\{a,b\}$. For $d=2$
there are two choices; for $d>2$ they are parametrized by $S^{d-2}$.
The orthogonal maps fixing the endpoint axis act transitively on
these choices. In the chosen plane $b=\pm e$, and only $b=e$ gives
the upper trajectory since $z>0$. 
\end{proof}

\paragraph{Local minimizers on either side.}
\label{app:homotopy-class-minimizers}
The local minimizers mentioned in \cref{sec:zero-endpoint-velocities}
follow from the next lemma.

Let $d=2$. A feasible trajectory $r$ has a continuous angle $\theta$ with
$r(t)=\norm{r(t)}(\cos\theta(t),\sin\theta(t))$, unique up to an additive
multiple of $2\pi$, so its total angle $\theta(T)-\theta(0)$ is well
defined. Feasible trajectories with a common total angle form a
\emph{homotopy class}; the two ways around the disk are two such classes,
whose total angles differ by $2\pi$.

 \begin{lemma}[Minimizers in a homotopy class]\label{lem:homotopy-class-minimizers}
If $\norm{r_T}>R$, every homotopy class of
feasible trajectories contains a minimizer of $\frac12\norm{\ddot r}^2$
over that class. Each such minimizer is a local minimizer
of~\eqref{eq:reduced} and satisfies~\eqref{eq:euler-lagrange} for a finite
nonnegative Radon measure $\mu$. \end{lemma}  

\begin{proof}[Proof of \cref{lem:homotopy-class-minimizers}]\sloppy
 Take a feasible $q$ (\cref{sec:problem_setup}) and choose
$\beta\in C^\infty([0,T];\mathbb R)$
equal to zero near $0$ and one near $T$. If $\operatorname{Rot}(\phi)$
denotes planar rotation by $\phi$, then, for $k\in\mathbb Z$,
$q_k(t)=\operatorname{Rot}(2\pi k\beta(t))q(t)$ is feasible, agrees
with $q$ near both endpoints, and adds $2\pi k$ to its total angle.
Thus every class is nonempty.

If feasible $q,r$ satisfy $\sup_{t\in[0,T]}\norm{q(t)-r(t)}<R$,
then $\inner{q(t),r(t)}>0$. Their relative angle lies in
$(-\pi/2,\pi/2)$ and vanishes at both endpoints, so they belong to
the same class. A minimizing sequence in a class is bounded in
$H^2((0,T);\mathbb R^2)$ by its acceleration cost and fixed initial
data. A subsequence converges weakly in $H^2((0,T);\mathbb R^2)$
and strongly in $C^1([0,T];\mathbb R^2)$; its limit is feasible and
remains in that class. Weak lower semicontinuity gives a minimizer.
Class stability and the Sobolev embedding make it a local minimizer,
and \cref{prop:relative-multipliers} supplies its multiplier measure. 
\end{proof}

\subsubsection{Proof of \cref{thm:canonical-catalogue}}
\label{app:canonical-exhaustion}

The counting argument concerns geometric contacts, including
zero-mass touches; the contact labels are assigned
afterward.

\paragraph{Contact order and curvature.}
\label{app:proof-canonical-contact-order}
 \begin{lemma}[Canonical contact order]\label{lem:canonical-contact-order}
For blocked data the interior geometric contacts of the selected minimizer
occur on the near side, strictly ordered counterclockwise, with
$w(t)\coloneqq r(t)\times \dot{r}(t)>0$ and $\tau(t)\coloneqq r(t)\times r^{(3)}(t)<0$
at each of them and $\tau$ continuous across an atom. \end{lemma}  \begin{proof}\sloppy Write $z(t)=\inner{ e,r(t)}$, so that every interior contact normal $n$ has
$\inner{ e,n}>0$. By \cref{prop:canonical-reaction,prop:side-selection},
the interior multiplier is finite, nonzero and compactly supported, and $z>0$
in $(0,1)$. Since the
free projection is the constant $\delta$, the Green reconstruction
\eqref{eq:reaction-green-reconstruction} gives, for $0<t<1$,
$u(t)\coloneqq z(t)-\delta
=\int_{0^+}^{1^-}G_1(t,s)z(s)\dd\mu(s)>0$,
$D_t^4u=z\,\mu\geq0$.
Thus $\ddot{u}$ is convex. Differentiating the kernel at the endpoints gives
$\ddot{u}(0)=\int_{0^+}^{1^-}s(1-s)^2z(s)\dd\mu(s)>0$,
$\ddot{u}(1)=\int_{0^+}^{1^-}s^2(1-s)z(s)\dd\mu(s)>0$.
Its integral is $\dot{u}(1)-\dot{u}(0)=0$, so $\ddot{u}$ is negative between
exactly two zeros and positive outside them. Integrating once,
$\dot{z}$ starts positive, decreases strictly across this negative
interval, and ends negative before reaching $\dot{z}(1)=0$. It therefore
has exactly one interior zero $t_*$, with $\ddot{z}(t_*)<0$.

Put $f=(d_0-r_{T,1},-r_{T,2})/\norm{r_T-r_0}$ and $X(t)=\inner{ f,r(t)}$.
The frame $(f,e)$ is positively oriented. Blockage of the endpoint
segment implies $\delta<1$, $X(0)>0$ and $X(1)<0$.
At any crossing $X=0$, feasibility and $z>0$ require $z\geq1$.
An interior contact before $t_*$ must have $X>0$: if $X<0$,
an earlier crossing would have height at least one, contradicting
strict increase of $z$ up to a contact with $z\leq1$; if $X=0$,
tangency would give $\dot{z}=0$. Similarly every contact after $t_*$
has $X<0$. At these contacts tangency gives $\dot{z}=wX$, hence $w>0$.

If $t_*$ itself is a contact, the existence of a crossing and the
maximality of $z(t_*)$ imply $z(t_*)=1$ and $X(t_*)=0$.
It is the only crossing, from positive to negative $X$.
For $b(t)=\norm{r(t)}^2-1$, feasibility gives
$0\leq \ddot{b}(t_*)=2(w(t_*)^2+\ddot{z}(t_*))$, so $w(t_*)\neq0$.
The crossing direction gives $\dot{X}(t_*)=-w(t_*)<0$.
The strict increase of $z$ on the right half of the circle and
its strict decrease on the left half now strictly order all
interior contact normals in the open semicircle $z>0$.

Fix any interior geometric contact $t$, write $n=r(t)$,
$e_t=(-n_2,n_1)$ and
$\tau=\inner{ e_t,r^{(3)}(t^\pm)}$. The last value is independent
of the side because jerk jumps are radial. Integrating
$D_t^1 r^{(3)}=r\,\mu$ to both endpoints with zero velocity gives
$0=w(t)-t\inner{ e_t,\ddot{r}(t)}+\tfrac12t^2\tau+P_{\mathrm L}$ and
$0=w(t)+(1-t)\inner{ e_t,\ddot{r}(t)}+\tfrac12(1-t)^2\tau+P_{\mathrm R}$, where
$P_{\mathrm L}=-\int_{0^+}^{t^-}
\tfrac12s^2\inner{ e_t,r(s)}\dd\mu(s)\geq0$ and
$P_{\mathrm R}=\int_{t^+}^{1^-}
\tfrac12(1-s)^2\inner{ e_t,r(s)}\dd\mu(s)\geq0$.
The signs follow from the strict angular order on $\operatorname{supp}(\mu)$, whose angular span is below $\pi$. If $\tau\geq0$,
the first equality forces $\inner{ e_t,\ddot{r}(t)}>0$, while the second forces
$\inner{ e_t,\ddot{r}(t)}<0$. Hence $\tau<0$. \end{proof} 

\emph{Three-contact curvature.}
The following local inequality uses the ordered normals, contact
curvatures, and tangential-jerk signs supplied by
\cref{lem:canonical-contact-order}.

 \begin{lemma}[Three-contact curvature inequality]\label{lem:three-contact-curvature}
Let three distinct unit-circle normals $n_i=r(t_i)$,
$t_1<t_2<t_3$, occur in counterclockwise order with total angular
span below $\pi$. Suppose $r$ is cubic on each intervening interval,
has continuous position, velocity and acceleration at $t_2$, and
has tangent velocities $\dot{r}(t_i)=w_i(-n_{i,2},n_{i,1})$ with
$w_i>0$. Assume
$\gamma(t_i)\coloneqq w_i^2+\inner{ n_i,\ddot{r}(t_i)}\geq0$,
$\jmp{r^{(3)}}_{t_2}=m n_2$,
$m\geq0$,
and negative outside tangential jerks
$n_1\times r^{(3)}(t_1^+)<0$, $n_3\times r^{(3)}(t_3^-)<0$.
Then
$\gamma(t_2)<\min\{\gamma(t_1),\gamma(t_3)\}$. \end{lemma}  \begin{proof}\sloppy Rescale time to $(t_1,t_2,t_3)=(-1,0,h)$, $h>0$, and rotate
$n_2$ to $(1,0)$. Write $n_1=(C,-S)$,
$n_3=(D,V)$, with $S,V>0$, and $\alpha_{\mathrm n}=\ddot{r}_1(0)$.
Let $r_{\mathrm L},r_{\mathrm R}$ be the polynomial extensions of the
left and right cubics, and set
$b_{\mathrm L}(t)=\norm{r_{\mathrm L}(t)}^2-1$ and
$b_{\mathrm R}(t)=\norm{r_{\mathrm R}(t)}^2-1$.
The double zeros of $b_{\mathrm L}$ at $-1,0$, its endpoint second derivatives and
leading coefficient give
$b_{\mathrm L}(t)=t^2(t+1)^2
\bigl[a_6t^2+(a_6+\gamma(0)-\gamma(-1))t+\gamma(0)\bigr]$,
$a_6=\norm{r_{\mathrm L}^{(3)}(0)}^2/36$.
Suppose $\gamma(0)\geq\gamma(-1)$. The bracket has nonnegative
coefficients, not all zero: otherwise the polynomial curve would
lie on the circle identically and hence be constant, contrary to
the distinct normals. Consequently $b_{\mathrm L}(t)>0$ and
$\frac{\dd}{\dd t}(b_{\mathrm L}(t)/t^6)<0$ for $t>0$. For the latter assertion,
write the quotient as
$(1+t^{-1})^2[a_6+(a_6+\gamma(0)-\gamma(-1))t^{-1}+\gamma(0)t^{-2}]$.
If $m=0$, this already precludes the third contact.
For $m>0$, use $r_{\mathrm R}=r_{\mathrm L}+mt^3(1,0)/6$ to obtain
$\frac{\dd}{\dd t}\frac{b_{\mathrm R}-b_{\mathrm L}}{t^6}
=\frac m3\left(-\frac3{t^4}-\frac{\alpha_{\mathrm n}}{2t^2}\right)$.
The tangent contact at $h$ makes $\left.\frac{\dd}{\dd t}(b_{\mathrm R}(t)/t^6)\right|_{t=h}=0$;
therefore
$-\alpha_{\mathrm n} h^2>6$.
Hermite interpolation at the right contact gives
$\alpha_{\mathrm n} h^2=6(D-1)+2hVw_3$, hence $D<0$.
Thus $-1<D<0$ and $D^2+V^2=1$.
The total angular span gives $0<C<1$; write $A=1-C>0$.

\emph{Negative middle tangential acceleration.}
Set $u=hw_2$ and $v=hw_3$, both positive.
Since $\gamma(0)=w_2^2+\alpha_{\mathrm n}\geq0$,
$-\alpha_{\mathrm n} h^2>6$ implies $u^2>6$.
The right Hermite cubic also gives
$\ddot{r}_2(0)h^2=6V-4u-2Dv$,
$\frac{h^3}{6}(n_3\times r_{\mathrm R}^{(3)})=v+Du-2V<0$.
If $\ddot{r}_2(0)\geq0$, these two relations imply
$(2-D^2)u<V(3-2D)$. But $2-D^2=1+V^2$ and
$-2DV\leq D^2+V^2=1$, so
$u<\frac{3V+1}{1+V^2}<\frac94<\sqrt6$,
where the middle strict inequality follows from
$9(1+V^2)-4(3V+1)=(3V-2)^2+1>0$.
This contradicts $u^2>6$, proving $\ddot{r}_2(0)<0$.

\emph{Ordering the first two speeds.}
The left Hermite cubic gives
$\ddot{r}_2(0)=-6S+2Cw_1+4w_2$ and $\gamma(-1)=w_1^2+2Sw_2-6A$.
Suppose $w_2\geq w_1$. Then $\ddot{r}_2(0)<0$ implies
$0<w_1<W\coloneqq3S/(C+2)$. Define
$Q(z)=z^2-CSz-3A^2$. The identity
$\gamma(-1)=Q(w_1)+S\ddot{r}_2(0)/2$ gives $\gamma(-1)<Q(w_1)$.
This convex quadratic is negative throughout $[0,W]$, since
$Q(0)=-3A^2<0$ and $Q(W)=-\frac{3A^2}{(C+2)^2}<0$.
This contradicts $\gamma(-1)\geq0$, so $w_2<w_1$.

\emph{The common transverse cubic.}
We have
$\gamma(0)-\gamma(-1)\allowbreak=(w_2-w_1)(w_1+w_2-2S)$.
Our assumed $\gamma(0)\geq\gamma(-1)$ therefore gives
$w_1+w_2\leq2S$ and $w_2<S$.
The radial jump leaves the transverse jerk unchanged; write
$r_{\mathrm L,2}^{(3)}(0)=r_{\mathrm R,2}^{(3)}(0)=6(Cw_1+w_2-2S)$.
Using $\alpha_{\mathrm n}=2Sw_1-6A$, we obtain
$r_{\mathrm L,2}^{(3)}(0)-\alpha_{\mathrm n} w_2
=6C(w_1+w_2-2S)+12A(w_2-S)-2Sw_1w_2<0$.
Together with $\ddot{r}_2(0)<0$ and
$-\alpha_{\mathrm n} h^2>6$, this forces
$\frac{r_2(h)}h
=w_2+\frac{\ddot{r}_2(0)h}{2}+\frac{r_{\mathrm L,2}^{(3)}(0)h^2}{6}
<w_2\left(1+\frac{\alpha_{\mathrm n} h^2}{6}\right)<0$,
contrary to $r_2(h)=V>0$. Thus $\gamma(0)<\gamma(-1)$.

Time reversal followed by a reflection $F$ restores counterclockwise
order. The transformed jets are $-F\dot{r}$, $F\ddot{r}$, and $-Fr^{(3)}$.
Curvatures and positive oriented speeds are preserved, the middle
jump is $mF n_2$, and $(Fn)\times(-Fr^{(3)})=n\times r^{(3)}$.
The same argument therefore proves $\gamma(0)<\gamma(h)$. \end{proof} 

\paragraph{Angular dynamics on an arc.}
\label{app:proof-boundary-dynamics}
For motion in an oriented plane,
$\Gamma\coloneqq\dot r\times\ddot r-r\times r^{(3)}$
is the corresponding component of $\boldsymbol\Gamma$ in
\cref{prop:contact-invariants}.
On an arc write $r=(\cos\theta,\sin\theta)$ and $w=\dot\theta$.

 \begin{proposition}[Angular dynamics on a boundary arc]\label{prop:boundary-dynamics}
On a boundary arc $[a,b]$, write $r=(\cos\theta,\sin\theta)$ with
continuous $\theta$ and $w=\dot{\theta}$.
Then $\ddot{w}=2w^3-\Gamma$ on $(a,b)$, with first
integral $(\dot{w})^2=w^4-2\Gamma w+2E$. \end{proposition}  

\begin{proof}[Proof of \cref{prop:boundary-dynamics}]\sloppy
 The continuous arc density in \cref{lem:clearance-identity} gives
continuous fourth derivatives of $r$ and its angle in the arc interior.
With $n=r$ and $e_\theta=(-\sin\theta,\cos\theta)$, differentiation gives
$\dot r=we_\theta$, $\ddot r=\dot w e_\theta-w^2n$, and
$r^{(3)}=(\ddot w-w^3)e_\theta-3w\dot w n$.
Substitution into the invariants of \cref{prop:contact-invariants}
gives $\Gamma=2w^3-\ddot w$ and
$E=\tfrac12((\dot w)^2-w^4)+\Gamma w$, proving both identities. 
\end{proof}

With $\alpha=\dot w$, \cref{lem:clearance-identity,prop:boundary-dynamics} give
\refstepcounter{equation}\label{eq:canonical-arc-density}%
$\Lambda=w^4-3\alpha^2-4w\ddot{w}
=4\Gamma w-7w^4-3\alpha^2
=10\Gamma w-10w^4-6E$~\textup{(\theequation)}.
\paragraph{Contact bounds and arc signs.}
 \begin{proposition}[Geometric contact bounds]\label{prop:canonical-contact-bounds}
For the selected minimizer, either $\mathcal A\cap(0,1)$ has at most
three elements, or there exist $0<a<b<1$ and sets
$S_-\subseteq(0,a)$, $S_+\subseteq(b,1)$, each with at most one element,
such that $\mathcal A\cap(0,1)=S_-\cup[a,b]\cup S_+$. \end{proposition}  \begin{proof}\sloppy For unblocked data the free cubic has at most one interior geometric
contact. Otherwise \cref{prop:canonical-reaction,prop:contact-structure}
give finitely many contact components between free endpoint pieces,
ordered by \cref{lem:canonical-contact-order}. Write
$b(t)=\norm{r(t)}^2-1=-g(t)$.
At geometric contacts, $\gamma(t)=\ddot{b}(t)/2\geq0$;
at arc endpoints, $\gamma(t)=0$.

\emph{At most one arc.}
For a free bridge $[a_0,b_0]$ of duration $\ell$ between tangent
unit contacts, set $A_0=\ell^2\gamma(a_0)$,
$A_1=\ell^2\gamma(b_0)$, and
$\kappa=\ell^6\norm{r^{(3)}(a_0^+)}^2/36$.
The double endpoint zeros factor its gap as
\refstepcounter{equation}\label{eq:canonical-bridge-gap}%
$b(a_0+\ell u)=u^2(1-u)^2
[A_0(1-u)+A_1u-\kappa u(1-u)]$~\textup{(\theequation)}.
Differentiating gives
$b^{(3)}(b_0^-)-b^{(3)}(a_0^+)
=12(\gamma(a_0)+\gamma(b_0))/\ell
+\ell^3\norm{r^{(3)}(a_0^+)}^2/3>0$.
If the right side vanished, the preceding factorization would
force the polynomial curve to lie on the circle and hence be
constant, contrary to the distinct ordered normals. At every
intervening atom, $b^{(3)}$ jumps by $2m_i\geq0$, whereas it
vanishes on an arc. Thus no finite chain of free bridges and isolated contacts connects two arcs.

\emph{Counting geometric contacts.}
By \cref{lem:three-contact-curvature}, four consecutive isolated
interior contacts at $t_1<t_2<t_3<t_4$ would give both
$\gamma(t_2)<\gamma(t_3)$ and $\gamma(t_3)<\gamma(t_2)$.
Two isolated contacts at $t_1<t_2$ immediately before an arc with
entry $a$ would give $\gamma(t_2)<\gamma(a)=0$; time reversal
excludes two after an arc. The contact laws and
\cref{lem:canonical-contact-order} supply the lemma's hypotheses,
also at zero-mass contacts. \end{proof} 

 \begin{proposition}[Arc signs]\label{prop:canonical-arc-signs}
In the arc case of \cref{prop:canonical-contact-bounds}, the angular
speed $w\coloneqq r\times \dot{r}$ satisfies $w(t)>0$ for $t\in[a,b]$.
Moreover, $\Lambda(t)>0$ for $t\in(a,b)$ and $\mu(\{t\})>0$ for
every $t\in S_-\cup\{a,b\}\cup S_+$. \end{proposition}  \begin{proof}\sloppy Write the arc as $[t_{\mathrm{entry}},t_{\mathrm{exit}}]$ and set
$b(t)=\norm{r(t)}^2-1$.

\emph{Arc speed and density.}
Integrating the energy invariant by parts with zero initial and final velocities
gives $E=\frac32\int_0^1\norm{\ddot r(t)}^2\dd t>0$.
By \cref{lem:canonical-contact-order}, $w>0$ on the closed arc.
Use $\alpha=\dot w$, \cref{prop:boundary-dynamics}, and
\eqref{eq:canonical-arc-density}, with their continuous arc-side
endpoint limits.

The arc equations are smooth, so an interior zero of
$\Lambda\geq0$ would satisfy
$\dot{\Lambda}=10(\Gamma-4w^3)\alpha=0$ and $\ddot{\Lambda}\geq0$.
If $\alpha=0$, the zero-density equation gives
$\Gamma=7w^3/4$ and then $\ddot{\Lambda}=-45w^6/8<0$.
Otherwise $\Gamma=4w^3$, $\alpha^2=3w^4$ and
$\ddot{\Lambda}=-360w^6<0$. Thus the density is strictly positive
in the arc interior. Its arc-side endpoint traces are nonnegative
and need not be strictly positive.

\emph{Positive atoms at arc endpoints.}
Consider first an arc exit at time $t_{\mathrm{exit}}$ followed by a
geometric contact after a free bridge of duration $\ell$. Let $m\geq0$
be its exit atom mass and set
$\widehat b(u)=b(t_{\mathrm{exit}}+\ell u)$ for $0\leq u\leq1$.
Since $b=\dot{b}=\ddot{b}=0$ on the arc and
$b^{(3)}(t_{\mathrm{exit}}^+)=2m$, we have
$\frac{\dd^3\widehat b}{\dd u^3}(0^+)=2\ell^3m$.
The factorization~\eqref{eq:canonical-bridge-gap} gives
\refstepcounter{equation}\label{eq:catalogue-arc-bridge-gap}%
$\widehat b(u)=u^3(1-u)^2
\bigl[\ell^3m/3+u\ell^6\norm{r^{(3)}(t_{\mathrm{exit}}^+)}^2/36\bigr]$~\textup{(\theequation)}.
If $m=0$, the fourth derivative in physical time is
$b^{(4)}(t_{\mathrm{exit}}^+)
=2\ell^2\norm{r^{(3)}(t_{\mathrm{exit}}^+)}^2/3>0$.
Here $r^{(3)}(t_{\mathrm{exit}}^+)\neq0$, since otherwise the
factorization would again force a constant circle-valued
polynomial. But the continuous clearance forcing in
\eqref{eq:clearance-identity} gives this fourth derivative as
$f_g(t_{\mathrm{exit}})=-2\Lambda(t_{\mathrm{exit}}^-)\leq0$, a contradiction.
Time reversal and reflection prove the same assertion at entry.

It remains to treat a direct cubic completion with zero endpoint velocity,
including an endpoint on the circle. Orient this cubic from that endpoint to arc entry
at time $t_{\mathrm{entry}}$, put its duration $a_0>0$, and use the
contact frame at entry
$n=(1,0)$, $\dot{r}=(0,w)$, $\ddot{r}=(-w^2,\alpha)$, with $w>0$.
Set $\widehat a_0=a_0w$ and $\beta_{\mathrm{tail}}=a_0\alpha/w$. Integrating the cubic
back to zero velocity and matching tangential jerk gives
$Z_0=\left(1-\frac{\widehat a_0^2}{6},\frac{\widehat a_0(\beta_{\mathrm{tail}}-4)}6\right)$,
$m=\frac{w^2}{a_0}(2-3\beta_{\mathrm{tail}})$,
$\Gamma-w^3=\frac{2w(1-\beta_{\mathrm{tail}})}{a_0^2}$.
For $0\leq u\leq1$, define the normalized gap
$\widehat b_0(u)=b(t_{\mathrm{entry}}-a_0+a_0u)$. Then
$\widehat b_0(u)=\frac{\widehat a_0^2}{36}(1-u)^3Q_0(u)$.
In the standard cubic Bernstein basis
$\binom3i u^i(1-u)^{3-i}$, $i=0,1,2,3$, the coefficients of $Q_0$ are
$B_0$,
$2B_0$,
$3B_0+4(1-\beta_{\mathrm{tail}})$,
$12(2-3\beta_{\mathrm{tail}})$,
$B_0=\widehat a_0^2+\beta_{\mathrm{tail}}^2-8\beta_{\mathrm{tail}}+4
=\frac{36(\norm{Z_0}^2-1)}{\widehat a_0^2}$.
The exterior endpoint and $m\geq0$ give $B_0\geq0$ and
$\beta_{\mathrm{tail}}\leq2/3$. Hence $\widehat b_0(u)>0$ for $0<u<1$, even when
$\norm{Z_0}=1$. If $m=0$, then $\beta_{\mathrm{tail}}=2/3$ and $B_0\geq0$
requires $\widehat a_0^2\geq8/9$. On the other hand, substituting the
tangential matching relation in the arc density gives
$0\leq\Lambda(t_{\mathrm{entry}}^+)=\frac{w^2}{a_0^2}\left(\frac43-3\widehat a_0^2\right)$,
so $\widehat a_0^2\leq4/9$, a contradiction. Thus both arc endpoints carry positive atoms.

\emph{Isolated contacts beside the arc.}
There is at most one isolated geometric contact on each side
of the arc. If its mass were zero, its neighboring cubics
would have identical jets through jerk and would join into a
single cubic from the endpoint with zero velocity to the arc. That cubic would have an interior touch,
contrary to the strict clearance of $\widehat b_0$.
Thus any such isolated contact has positive mass and there
are no incidental zero-mass touches beside a persisting arc. \end{proof} 

\begin{proof}[Proof of \cref{thm:canonical-catalogue}]\sloppy
 The geometric configurations of \cref{prop:canonical-contact-bounds}
give F, P, PP, or PPP without an arc, and A, PA, AP, or PAP with one,
using the contact-label convention of \cref{sec:zero-endpoint-velocities}.
\Cref{prop:canonical-arc-signs} excludes zero-mass contacts beside
an arc. Multiplier uniqueness (\cref{prop:canonical-reaction}) fixes the label. 
\end{proof} 

\subsubsection{Proof of \cref{thm:endpoint-regime-map}}
\label{app:proof-endpoint-regime-map}
\label{app:canonical-atomic}
\label{app:planar-boundary-dynamics}
\label{app:canonical-completion}

\paragraph{Atomic construction and clearance.}

The free trajectory is feasible exactly when the endpoint segment avoids the open disk:
for $D=r_T-r_0\neq0$, test
$\norm{r_0+s_*D}\geq1$, where
$s_*=\min\{1,\max\{0,-\inner{ r_0,D}/\norm{D}^2\}\}$.
For $D=0$ the test succeeds.

\emph{Atomic representation.}
For blocked data write $G=G_1$.
The atomic contact equations in \cref{sec:canonical-constructions} imply
$\Delta_i=m_i n_i$, so four distributional derivatives of the Green
representation give
$D_t^4r=r\sum_i m_i\delta_{t_i}$.
The kernel enforces the endpoints and continuity through acceleration.
Admissibility supplies the hypotheses of \cref{thm:contact-certificate},
and \cref{prop:side-selection} identifies the selected minimizer.
Conversely, \eqref{eq:reaction-green-reconstruction} represents every
purely atomic selected minimizer, and
\cref{prop:canonical-reaction} determines its positive jumps and times uniquely.

On each time segment the Green branches are polynomials. Clearance
can be imposed without introducing a root selector as
$\forall u\quad
\bigl(t_i\leq u\leq t_{i+1}\ \Rightarrow\
\norm{r_i(u)}^2\geq1\bigr)$,
$t_0=0$,
$t_{m+1}=1$,
where $r_i$ is the corresponding cubic polynomial.
Impose $\inner{ e,n_i}>0$ at every atom. Multiply this
projection inequality by the positive
$\norm{r_T-r_0}$ to use the polynomial normal $(r_{T,2},d_0-r_{T,1})$.
On the ordered time simplex, the equations and inequalities thus
provide a real-algebraic specification in the $3m$ unknowns.
Retain all real components, with distinct times and positive masses
as required by the contact-label convention in
\cref{sec:zero-endpoint-velocities}.

\emph{Polynomial clearance tests.}
For a contact-to-contact bridge, test the quadratic factor
in~\eqref{eq:canonical-bridge-gap}. Its minimum occurs at
$u_*=\min\{1,\max\{0,(A_0+\kappa-A_1)/(2\kappa)\}\}$ when
$\kappa>0$; for $\kappa=0$, evaluate both endpoints. Outer pieces require only
the endpoints and the real roots of a cubic derivative after their
forced double zero is removed. A constant quotient is checked at one
point. At a terminal contact with zero velocity a further double factor is
available only on that endpoint stratum. These are exact tests of
nonnegativity, retaining zero-mass touches.

\paragraph{Arc system and completion.}

In unit-duration time, \eqref{eq:clearance-identity} gives the Cartesian
arc equation
$r^{(4)}(t)=\lambda(r(t),\dot{r}(t),\ddot{r}(t),r^{(3)}(t))r(t)$, where
$\lambda(z_0,z_1,z_2,z_3)\coloneqq-3\norm{z_2}^2-4\inner{z_1,z_3}$.
Here $\Lambda(t)=\lambda(r(t),\dot{r}(t),\ddot{r}(t),r^{(3)}(t))$ on the arc.

On an arc $[a,b]$, the entry speed
$w_a=\dot{\theta}(a)>0$ (\cref{prop:canonical-arc-signs}) defines
$s\coloneqq w_a(t-a)$ and
$\hat\theta(s)\coloneqq\theta(a+s/w_a)-\theta(a)$.
Dots on hatted quantities denote derivatives with respect to $s$.
Set $\hat w\coloneqq\dot{\hat\theta}$, $\hat\alpha\coloneqq\dot{\hat w}$,
$\hat\Gamma\coloneqq\Gamma/w_a^3$, $\hat E\coloneqq E/w_a^4$, and
$\hat T_{\mathrm A}\coloneqq w_a(b-a)$.
The normalized equations are
$\dot{\hat\theta}=\hat w$, $\dot{\hat w}=\hat\alpha$, and
$\dot{\hat\alpha}=2\hat w^3-\hat\Gamma$ on $[0,\hat T_{\mathrm A}]$.
For a choice of $(\hat\alpha_0,\hat\Gamma,\hat T_{\mathrm A})$ with
$\hat T_{\mathrm A}>0$, impose
$\hat\theta(0)=0$, $\hat w(0)=1$, and $\hat\alpha(0)=\hat\alpha_0$;
then $\hat E=\tfrac12(\hat\alpha_0^2-1+2\hat\Gamma)$ and the arc is
$\hat r(s)=(\cos\hat\theta(s),\sin\hat\theta(s))$.

On intervals where the angular speed is strictly monotone, the identity in
\cref{prop:boundary-dynamics} gives elapsed time by integrating with respect
to angular speed.
The raw density formula~\eqref{eq:canonical-arc-density} applies
to $\hat\Lambda$ with $w=\hat w$, $\alpha=\hat\alpha$,
$\Gamma=\hat\Gamma$, and $E=\hat E$.
Require $\hat w>0$ and $\hat\Lambda\geq0$ throughout the arc.

\emph{Completion relation.}

By \cref{prop:canonical-contact-bounds}, each end of an arc reaches its
endpoint with zero velocity either directly or through one isolated contact. Both
completions are determined by the jet at the arc end, up to finitely many
roots.
For an arc $[a,b]$, use local outward time $s$, with $s=0$ at the
chosen endpoint and $s<0$ on the arc. Write the local trajectory as
$r(s)$, with dots denoting derivatives in $s$.
The density traces $\Lambda(a^+)$ and $\Lambda(b^-)$ retain the original
arc coordinate. Choose the frame so that
$r(0)=n=(1,0)$,
$\dot{r}(0)=(0,w)$,
$\ddot{r}(0)=(-w^2,\alpha)$,
$r_A^{(3)}=(-3w\alpha,\tau_{\mathrm A})$,
$\tau_{\mathrm A}\coloneqq w^3-\Gamma<0$.
Here $r_A^{(3)}=r^{(3)}(0^-)$ is the arc-side jerk.
At entry, $\Lambda(a^+)=-4\tau_{\mathrm A}w-3w^4-3\alpha^2$;
at exit, the same expression equals $\Lambda(b^-)$.
Nonnegative arc-side density and $w>0$ imply $\tau_{\mathrm A}<0$.
For the normalized arc, set
$\Gamma=\hat\Gamma$. At entry use $w=1$, $\alpha=-\hat\alpha_0$
and reflect the resulting positions in the local horizontal axis;
at exit use $w=\hat w(\hat T_{\mathrm A})$, $\alpha=\hat\alpha(\hat T_{\mathrm A})$ and rotate by
$\hat\theta(\hat T_{\mathrm A})$. The invariant $\Gamma$ is unchanged by the combined
time reversal and reflection.

\emph{The stopping relation.}
Every completion ends with a free cubic from its last contact to the endpoint with zero velocity.
At the last contact, whose local time is $t_c$, write $p=r(t_c)$,
$\dot{r}_c=\dot{r}(t_c)$, $\ddot{r}_c=\ddot{r}(t_c)$, and
$r_-^{(3)}=r^{(3)}(t_c^-)$, $r_+^{(3)}=r^{(3)}(t_c^+)$.
Set $w_c=p\times\dot{r}_c>0$,
$\alpha_c=p\times\ddot{r}_c$, and
$\tau_c=p\times r_-^{(3)}<0$.
The final cubic's duration, endpoint, and outgoing jerk are
\refstepcounter{equation}\label{eq:canonical-tail-completion}%
$T_{\mathrm{tail}}=2w_c/(\sqrt{\alpha_c^2-2\tau_cw_c}-\alpha_c)$,
$Z=p+\frac{2T_{\mathrm{tail}}}{3}\dot{r}_c+\frac{T_{\mathrm{tail}}^2}{6}\ddot{r}_c$, and
$r_+^{(3)}=-2\ddot{r}_c/T_{\mathrm{tail}}-2\dot{r}_c/T_{\mathrm{tail}}^2$~\textup{(\theequation)}.
The outgoing cubic reaches zero velocity after time $T_{\mathrm{tail}}$ precisely when its jerk
has this value. Tangential jerk continuity gives
$\tau_cT_{\mathrm{tail}}^2+2\alpha_cT_{\mathrm{tail}}+2w_c=0$, whose unique positive root is
\eqref{eq:canonical-tail-completion}. Integration gives $Z$.
The jump $\jmp{r^{(3)}}_{t_c}$ is radial, with positive atom mass
$\inner{ p,\jmp{r^{(3)}}_{t_c}}$.
For direct completion from the arc to the endpoint with zero velocity, $p=n$ and $r_-^{(3)}=r_A^{(3)}$, giving
$m=\inner{ n,\jmp{r^{(3)}}_{t_c}}=2w^2/T_{\mathrm{tail}}+3w\alpha$.

\emph{An optional isolated contact.}
Before it stops, the trajectory may touch the disk once more.
Launch $r(s)=n+s\dot{r}(0)+s^2\ddot{r}(0)/2+s^3r_B^{(3)}/6$,
where $r_B^{(3)}=r^{(3)}(0^+)=r_A^{(3)}+mn$ is the bridge jerk.
For a contact at $s=\ell>0$, set
$p_2=w\ell+\frac{\alpha\ell^2}{2}+\frac{\tau_{\mathrm A}\ell^3}{6}$,
$A_\ell=3-\frac{w^2\ell^2}{2}$, and
$B_\ell=3-\left(2w\ell+\frac{\alpha\ell^2}{2}\right)p_2$,
and define
${P_{\mathrm{con}}(\ell)=144[B_\ell^2+A_\ell^2(p_2^2-1)]/\ell^4}$.
The quotient extends to a sextic polynomial with leading coefficient
$\tau_{\mathrm A}^2(w^4+\alpha^2)>0$.
Indeed, if $p=r(\ell)=(p_1,p_2)$, direct substitution gives
$\ell\inner{ p,\dot{r}(\ell)}-3(\norm{p}^2-1)
=B_\ell-A_\ell p_1$.
Consequently unit position and tangency are equivalent to
$A_\ell p_1=B_\ell$ and $p_1^2+p_2^2=1$.
Ordered normals require $p_2>0$, so retain every root of $P_{\mathrm{con}}$
in the explicit bracket
${0<\ell<
\frac{3\alpha+\sqrt{9\alpha^2-24\tau_{\mathrm A}w}}{-2\tau_{\mathrm A}}}$.
For $A_\ell\neq0$ set $p_1=B_\ell/A_\ell$.
For $A_\ell=0$ require $B_\ell=0$ and $0<p_2\leq1$, and
enumerate both real values $p_1=\pm\sqrt{1-p_2^2}$,
deduplicating zero. In either case set
$m=\frac{6(p_1-1+w^2\ell^2/2)}{\ell^3}+3w\alpha$
and restore the original unit-position and tangency equations.
There are at most six distinct durations. Since $A_\ell=0$ has
only one positive solution, there are at most seven position-duration
candidates, including that exceptional case.
Keep $m>0$, positive tangent speed and negative tangent jerk at $p$,
and attach its unique final cubic with zero endpoint velocity using
\eqref{eq:canonical-tail-completion}. Retain the admissible
completions, including all admitted multiple roots.

\emph{Clearance.}
Equation~\eqref{eq:catalogue-arc-bridge-gap} and the factorization for a cubic
completion in the proof of \cref{prop:canonical-arc-signs} give
strict clearance inside admissible bridges and direct cubic completions.
An outer tail following P still requires the polynomial test above.

\paragraph{Representation completeness.}
For fixed $(\hat\alpha_0,\hat\Gamma,\hat T_{\mathrm A})$ there are at most eight raw
completion choices per side, including the direct choice.
After joining and filtering, let $s_0$ be the initial temporary time
and $L$ the total duration, so the assembled trajectory $r_{\rm tmp}$
is defined on $[s_0,s_0+L]$.
Then $r_{\rm unit}(t)=r_{\rm tmp}(s_0+Lt)$, $t\in[0,1]$,
normalizes the full trajectory to unit duration.
Velocities, accelerations, jerk
and atom masses acquire factors $L,L^2,L^3,L^3$, respectively;
the density acquires $L^4$. Endpoints are unchanged.
For endpoints $X,Y$ in the construction frame, first require
$\norm{X}>1$, $\norm{Y}\geq1$, and $X\times Y\geq0$. Rotation of $X$ onto
the positive axis then gives the datum
$\mathcal E_{\mathrm{can}}(X,Y)=
  \left(\norm{X},\frac{\inner{ X,Y}}{\norm{X}},
                    \frac{X\times Y}{\norm{X}}\right)$.
Retain only admissible constructions as defined in
\cref{sec:canonical-constructions}, matching the datum when endpoints
are prescribed. The matched jets give $D_t^4r=r\,\mu$ with finite
nonnegative multiplier compactly supported in the time interval.
\Cref{thm:contact-certificate,prop:side-selection} identify the selected optimum.

Conversely, take a selected optimum with arc $[a,b]$, angular speed
$w(t)$, and entry speed $w_a=w(a)>0$. Rotate its entry normal to $(1,0)$ and use
$s=w_a(t-a)$. Its normalized initial conditions are those of
\cref{app:planar-boundary-dynamics}, with
$\hat\alpha_0=\dot{w}(a)/w_a^2$, $\hat\Gamma=\Gamma/w_a^3$,
and $\hat T_{\mathrm A}=w_a(b-a)>0$.
\Cref{lem:canonical-contact-order,prop:canonical-contact-bounds,prop:canonical-arc-signs}
supply the order, signs, and allowed contact counts. The final cubic durations are
the unique positive roots in~\eqref{eq:canonical-tail-completion}.
Any intervening isolated contact satisfies
$P_{\mathrm{con}}(\ell)=0$ and its bracket, so the completion rules
retain it, including exceptional and multiple roots.
The initial temporary time is $s_0=-w_a a$
and the total duration is $L=w_a$, so the time
normalization and $\mathcal E_{\mathrm{can}}$ recover the
original datum and trajectory.

The plane basis of \cref{sec:zero-endpoint-velocities} gives
$U=(e_1\ e_2)$; \cref{rem:nondimensionalization} restores physical
trajectories by $r_{\rm phys}(t)=RUr_{\rm can}(t/T)$.
Varying the permitted basis recovers every minimizer by
\cref{prop:side-selection}. 

\begin{proof}[Proof of \cref{thm:endpoint-regime-map}]\sloppy
 Existence and \cref{thm:canonical-catalogue}, together with the
representation and certification proved above, give coverage and identify
every admitted construction with the selected minimizer.
If two constructions have the same datum, their trajectories therefore
agree. By \cref{prop:canonical-reaction}, their interior multipliers agree
as well, so their minimal labels coincide. Hence the regions are
pairwise disjoint. 
\end{proof}

\subsubsection{Bounds and correctness of \cref{alg:canonical-recovery}}
\label{app:canonical-recovery}

Represent the algebraic inputs by polynomials and isolating intervals.

\paragraph{Explicit bounds.}
 \begin{proposition}[Explicit bounds for canonical arc recovery]\label{prop:arc-recovery-bounds}
For blocked data, set $D=r_T-r_0$,
$U=162\bigl[(1+r_{T,2})^2+\allowbreak(r_{T,1}-d_0)^2+1\bigr]$,
$t_{\mathrm{free}}=\min\left\{\frac12,
\left(\frac{3(d_0-1)^2}{8U}\right)^{1/3}\right\}$,
$J_0=\sqrt{24U/t_{\mathrm{free}}^3}$, $B=d_0J_0$,
$w_{\min}=\frac{54\norm D^2}{5B}$,
$V=\sqrt{2U}$, $A=V^2+\sqrt{V^4+6U}$,
$C_\tau=B+VA$, $\delta_{\mathrm{end}}=\min\{1/4,w_{\min}^2/(4U)\}$, and
$C_\mu=\frac{4}{5t_{\mathrm{free}}^4\delta_{\mathrm{end}}^4}+2(A+C_\tau/2)^2+\frac{2A}{\delta_{\mathrm{end}}}$.
If the selected optimum contains an arc, then for every
$t\in[0,1]$ it satisfies
$\frac12\norm{\ddot r}^2\leq U$,
$0<\Gamma\leq B$,
$\norm{\dot{r}(t)}\leq V$,
$\norm{\ddot{r}(t)}\leq A$,
$\mu((0,1))\leq C_\mu$,
$\norm{r^{(3)}}_{L^\infty((0,1);\mathbb R^2)}\leq J_0+C_\mu$.
The initial free interval contains $[0,t_{\mathrm{free}}]$, and
$\norm{\ddot{r}(0)}\leq\sqrt{6U}$, $\norm{r^{(3)}(0)}\leq J_0$.
On the arc, $w(t)\geq w_{\min}$. \end{proposition}  \begin{proof}\sloppy Join $r_0$, $(d_0,1+r_{T,2})$, $(r_{T,1},1+r_{T,2})$, and $r_T$ by three straight
segments, each traversed in time $1/3$ with cubic timing and zero initial
and final velocities. The first segment has first coordinate $d_0>1$ and the second has
height at least one. On the third, either $\abs{r_{T,1}}\geq1$ or the
height is at least $r_{T,2}\geq\sqrt{1-r_{T,1}^2}$. This is a feasible
comparison trajectory with cost $U$, since a cubic segment with displacement $z$,
duration $\ell$, and zero initial and final velocities costs $6\norm z^2/\ell^3$.
The free lower bound and the energy identity for zero initial and final velocities give
$18\norm D^2\leq E=\frac32\norm{\ddot r}^2\leq3U$.
Blocked data have $D\neq0$.

Cauchy--Schwarz and $\dot{r}(0)=0$ give
$\norm{r(t)-r_0}\leq\sqrt{2U/3}\,t^{3/2}$ and
$\norm{\dot{r}(t)}\leq V$. Thus no contact occurs on $[0,t_{\mathrm{free}}]$.
There $\ddot{r}(t)=\ddot{r}(0)+t r^{(3)}(0)$, and completing the square gives
$2U\geq\int_0^{t_{\mathrm{free}}}\norm{\ddot{r}(0)+t r^{(3)}(0)}^2\dd t
\geq\frac{t_{\mathrm{free}}^3}{12}\norm{r^{(3)}(0)}^2$.
Since $\dot r(0)=0$, $\Gamma=-r_0\times r^{(3)}(0)$, hence $\abs{\Gamma}\leq B$.
On an arc, positive speed and nonnegative density imply
$\Gamma w\geq w^4+3E/5>0$, hence $\Gamma>0$ and
$w\geq\frac{3E}{5\Gamma}\geq w_{\min}$.

At a contact $t$, write $n=r(t)$, $\dot{r}(t)=w(-n_2,n_1)$.
The energy and momentum invariants give
$\norm{\ddot{r}(t)+w^2n}^2=w^4-2\Gamma w+2E$,
$\tau_n\coloneqq n\times r^{(3)}(t)=-\Gamma-w\inner{ n,\ddot{r}(t)}$.
Contact order gives $w>0$, so $\norm{\ddot{r}(t)}\leq A$ and
$\abs{\tau_n}\leq C_\tau$. Since the initial and final velocities are zero, both endpoint accelerations have
norm $\sqrt{2E}\leq\sqrt{6U}\leq A$. Acceleration is affine
on every free piece, proving the global acceleration bound.

At the arc exit $b$,
$w_{\min}\leq\norm{\dot{r}(b)}\leq\sqrt{2U(1-b)}$, so $b\leq1-2\delta_{\mathrm{end}}$.
For $\phi(t)=t^2(1-t)^2$, apply \cref{thm:contact-certificate} to
$h=\phi e_1$, using the projection established in
\cref{app:proof-planar-reduction}. This gives
$\int_{0^+}^{1^-}\phi(t)^2\dd\mu(t)
\leq\int_0^1\abs{\ddot{\phi}(t)}^2\dd t=\frac45$.
Since $\phi\geq t_{\mathrm{free}}^2\delta_{\mathrm{end}}^2$ on $[t_{\mathrm{free}},1-\delta_{\mathrm{end}}]$, the mass of $\mu$
there is at most $4/(5t_{\mathrm{free}}^4\delta_{\mathrm{end}}^4)$. By
\cref{thm:canonical-catalogue}, any mass of $\mu$ outside this interval lies in a
single isolated atom at $t_P>1-\delta_{\mathrm{end}}$.
The preceding free bridge starts at $b\leq1-2\delta_{\mathrm{end}}$, so its
constant jerk satisfies
$\norm{r_{\mathrm L}^{(3)}}\leq2A/\delta_{\mathrm{end}}$.

Set $T_{\mathrm{tail}}=1-t_P$, $n=r(t_P)$ and $w=n\times \dot{r}(t_P)$;
then $T_{\mathrm{tail}}<\delta_{\mathrm{end}}\leq1/4$.
The zero final tangential velocity gives
$w=-T_{\mathrm{tail}}(n\times \ddot{r}(t_P))-T_{\mathrm{tail}}^2\tau_n/2$, hence
$w\leq T_{\mathrm{tail}}(A+C_\tau/2)$. The contact curvature gives
$\inner{ n,\ddot{r}(t_P)}\geq-w^2$. Therefore its mass obeys
$m_P=-\frac{2\inner{ n,\ddot{r}(t_P)}}{T_{\mathrm{tail}}}
-\inner{ n,r_{\mathrm L}^{(3)}}
\leq\frac{2w^2}{T_{\mathrm{tail}}}+\frac{2A}{\delta_{\mathrm{end}}}
\leq2(A+C_\tau/2)^2+\frac{2A}{\delta_{\mathrm{end}}}$.
Summing proves the bound $C_\mu$.
Finally $D_t^1 r^{(3)}=r\,\mu$ and $\norm{r(t)}=1$ on its support yield
$\norm{r^{(3)}(t)}\leq\norm{r^{(3)}(0)}+\mu((0,1))\leq J_0+C_\mu$. \end{proof} 

\paragraph{Correctness and termination.}
The free test and \cref{alg:atomic-recovery}
terminate by real quantifier elimination and sample-point recovery
\citep[Section~2]{basu_survey_2017}:
the coefficients and clearance conditions in \cref{app:canonical-atomic}
are algebraic. If these tests fail, \cref{thm:canonical-catalogue}
leaves the four arc families.

\emph{Bounded parameters.}
For each arc family, use $m=2,3,$ or $4$ event times
$0\leq t_1\leq\cdots\leq t_m\leq1$, with consecutive arc endpoints,
initial acceleration and jerk $\ddot r(0),r^{(3)}(0)\in\mathbb R^2$,
and nonnegative atom masses $m_i$.
From $(r_0,0)$, propagate cubic pieces with jumps
$r^{(3)}\mapsto r^{(3)}+m_i r(t_i)$.
Propagate the arc between its entry and exit jumps.
Use rational upper bounds from \cref{prop:arc-recovery-bounds}
for the initial jets and masses. The four closed parameter domains,
including zero durations and masses, form a compact disjoint union
$\mathcal P_{\mathrm{rec}}$ with computable finite rational covers.

In the arc system of \cref{app:planar-boundary-dynamics},
choose rational $H>\max\{1,V,A,J_0+C_\mu\}$ and clip each state
coordinate to $[-H,H]$ only when evaluating the last component
$\lambda r$. This leaves true arc states unchanged and gives a
globally Lipschitz vector field, with constant
$L_H=\max\{1,42H^2\}$ in the maximum coordinate norm.

For $p\in\mathcal P_{\mathrm{rec}}$, write $q_p$ for the assembled trial
and $z_p$ for its one-sided arc state. Euler error estimates for the
clipped field, which vanishes at zero, and the polynomial cubic and
jump maps give computable state bounds and uniform parameter and time
moduli. Moving knots preserve continuity of position, velocity, and
acceleration, whose one-sided derivatives remain bounded.
Thus $p\mapsto q_p\in C^2([0,1];\mathbb R^2)$ is computable with a
computable uniform modulus for
$\norm q_{C^2([0,1];\mathbb R^2)}
\coloneqq\max_{0\leq\ell\leq2}\sup_{0\leq t\leq1}\norm{q^{(\ell)}(t)}$.
Integrating the squared acceleration gives the same properties for
$p\mapsto\frac12\norm{\ddot q_p}^2$.

\emph{Admissible parameters.}
Impose $q_p(1)=r_T$, $\dot q_p(1)=0$, and $\norm{q_p(t_i)}=1$;
require $\norm{q_p(t)}\geq1$ and $\inner{e,q_p(t)}\geq\delta$
throughout $[0,1]$; and on the closed arc require
$\norm{q_p(t)}=1$, $\lambda(z_p(t))\geq0$, and $z_p(t)\in[-H,H]^8$.
Use~\eqref{eq:canonical-projection}, with $e=(0,1)$, $\delta=0$
for symmetric data. At coincident events apply jumps in template order;
arc tests use the state after entry and before exit, including the
single state of a zero-duration arc.

These conditions define $\mathcal Z=\mathcal R^{-1}(0)$, where
$\mathcal R\geq0$ sums squared endpoint and event errors, suprema of
positive inequality violations, and the supremum of the absolute arc
circle error. After parametrizing variable intervals by $u\in[0,1]$,
finite meshes and the uniform state and time moduli compute
$\mathcal R$ and a uniform modulus.

Representation completeness and \cref{prop:arc-recovery-bounds} give
$\mathcal Z\neq\varnothing$. Every $p\in\mathcal Z$ is feasible and
satisfies $D_t^4q_p=q_p\mu_p$, where
$\mu_p=\sum_{0<t_i<1}m_i\delta_{t_i}+\lambda(z_p(t))\dd t$
has its density restricted to the arc and is nonnegative and supported
on the circle. Omit endpoint masses and combine coincident ones.
For $\delta>0$, \cref{thm:contact-certificate} applies directly.
For symmetric data, $z(t)=\inner{e,q_p(t)}\geq0$ satisfies
$z(t)=\int_{0^+}^{1^-}G_1(t,s)z(s)\dd\mu_p(s)$.
If $z\mu_p=0$, then $z=0$, which cannot clear the disk between
blocked symmetric endpoints. Otherwise positivity of $G_1$ gives
$z>0$ in $(0,1)$, so the certificate identifies the selected upper
optimum. Hence $\{q_p:p\in\mathcal Z\}=\{r_*\}$.
A zero-duration arc would give at most three combined interior atoms,
contradicting the exact atomic rejection.

\emph{Search and termination.}
Rational sample evaluations and the uniform moduli give the bounds on
parameter boxes in \cref{alg:arc-bounds}. Every zero remains covered, and
$\mathcal R$ tends uniformly to zero on retained boxes as resolution
and evaluation errors vanish. Time meshes, derivative bounds, and
the same moduli give trajectory and cost enclosures with errors
tending to zero.
Compactness places every limit of retained parameters in $\mathcal Z$.
Its trajectory image is $\{r_*\}$, so continuity forces the enclosure
diameters to zero. Thus \cref{alg:canonical-recovery} terminates with
bounds containing $r_*$ and its cost; the center trial need not be
feasible. This is effective compactness with a unique image
\citep[Fact~3.25 and the proof of Corollary~4.11]{downey_melnikov_2023}.

\end{document}